\documentclass{article}

\usepackage[final]{corl_2020}

\usepackage{amsmath}
\usepackage{hyperref}
\usepackage{amsthm}
\usepackage{amsfonts}
\usepackage{color}
\usepackage{array}
\usepackage{graphicx}
\graphicspath{ {./figs/} }
\usepackage{wrapfig}
\usepackage{mathtools}
\usepackage{amssymb}
\usepackage{floatrow}
\usepackage[ruled,noend]{algorithm2e}
\usepackage{wasysym}
\usepackage{todonotes}
\usepackage{tabularx}
\usepackage{enumerate}
\usepackage{diagbox}
\usepackage{setspace}

\newcolumntype{b}{X}
\newcolumntype{m}{>{\hsize=.8\hsize}X}
\newcolumntype{s}{>{\hsize=.5\hsize}X}

\newtheorem{theorem}{Theorem}
\newtheorem{problem}{Problem}

\newcommand{\task}{\sigma}

\newcommand{\traj}{\xi}
\newcommand{\state}{x}
\newcommand{\statespace}{\mathcal{X}}
\newcommand{\safeset}{\mathcal{S}}
\newcommand{\unsafeset}{\mathcal{A}}

\newcommand{\numsafe}{N_\textrm{dem}}

\newcommand{\control}{u}
\newcommand{\controlset}{\mathcal{U}}

\newcommand{\constraintset}{\mathcal{C}}

\newcommand{\trajxu}{\traj_{xu}}
\newcommand{\trajx}{\traj_\state}
\newcommand{\traju}{\traj_\control}
\newcommand{\taskspace}{\Sigma}
\newcommand{\constraintspace}{\mathcal{C}}

\newcommand{\guarunsafe}{\mathcal{G}_{\neg s}}
\newcommand{\guarsafe}{\mathcal{G}_s}
\newcommand{\feas}{\mathcal{F}}
\newcommand{\cstate}{\kappa}

\newcommand{\demj}{\traj_{j}^\textrm{loc}}

\newcommand{\thetac}{\gamma}
\newcommand{\Thetac}{\Gamma}

\newcommand{\eq}{\textrm{eq}}
\newcommand{\ineq}{\textrm{ineq}}
\newcommand{\tbox}{\textrm{box}}
\newcommand{\feastheta}{\feas_\theta}
\newcommand{\feasgamma}{\feas_\gamma}
\newcommand{\feasrobust}{\hat\feas_\theta}
\newcommand{\pr}{\textrm{Pr}}
\newcommand{\safe}{\textrm{ safe}}
\newcommand{\dem}{\textrm{dem}}
\newcommand{\policy}{\pi}
\newcommand{\collset}{\mathfrak{C}}
\newcommand{\lag}{\mathcal{L}}

\newcommand{\feasthetasun}{\feas_\theta^{s, \neg s}}

\makeatletter
\renewcommand{\fnum@figure}{\small Fig. \thefigure}
\makeatother
\usepackage{chngcntr}
\usepackage{apptools}
\AtAppendix{\counterwithin{lemma}{section}}
\AtAppendix{\counterwithin{assumption}{section}}
\AtAppendix{\counterwithin{theorem}{section}}
\AtAppendix{\counterwithin{corollary}{section}}

\title{Uncertainty-Aware Constraint Learning for \\Adaptive Safe Motion Planning from Demonstrations}

\author{
  Glen Chou, Necmiye Ozay, and Dmitry Berenson\\
  Department of Electrical Engineering and Computer Science\\
  University of Michigan, Ann Arbor\\
  \texttt{\{gchou, necmiye, dmitryb\}@umich.edu} \\
}

\begin{document}
\maketitle

\begin{abstract}
    We present a method for learning to satisfy uncertain constraints from demonstrations. Our method uses robust optimization to obtain a belief over the potentially infinite set of possible constraints consistent with the demonstrations, and then uses this belief to plan trajectories that trade off performance with satisfying the possible constraints. We use these trajectories in a closed-loop policy that executes and replans using belief updates, which incorporate data gathered during execution. We derive guarantees on the accuracy of our constraint belief and probabilistic guarantees on plan safety. We present results on a 7-DOF arm and 12D quadrotor, showing our method can learn to satisfy high-dimensional (up to 30D) uncertain constraints, and outperforms baselines in safety and efficiency.
\end{abstract}

\keywords{\hspace{-1.6pt}learning from demonstration, \hspace{-0.5pt}planning under uncertainty, \hspace{-0.5pt}safe learning} 

\section{Introduction}

Learning from demonstration (LfD) is a powerful framework for teaching robots to perform tasks. In particular, recent work has shown that modeling tasks as constrained optimizations, and learning the cost function and constraints from demonstrations \cite{ral, wafr, menner, toussaint}, can enable the learning of complex manipulation and mobile robotics tasks. However, a core problem in LfD, and constraint-learning in particular, is the unidentifiability of the constraints: there is often an infinite set of \textit{possible constraints} which are sufficient to explain a demonstration. While previous work \cite{ral} has evaded this problem when planning with the learned constraint by planning guaranteed-safe trajectories that satisfy \textit{all possible constraints} consistent with the data, this is impossible in most realistic scenarios, where the set of possible constraints is so large that the planning problem becomes infeasible. For example, consider planning for an arm in a cluttered home environment: unless the demonstrations activate each of the multitude of constraints, we cannot claim that a trajectory is guaranteed-safe.

Our insight is that to plan under large constraint uncertainty, it is vital to plan trajectories that trade off safety and efficiency by reasoning over the set of possible constraints, and to update this set using constraint information gathered when executing these trajectories. Specifically, we leverage robust optimization and Bayesian inference to obtain and update our belief over the possible constraints consistent with the demonstrations and gathered data. Then, we propose a policy for adaptively satisfying the constraints which interleaves chance-constrained planning, execution, and belief updates until the task is completed. This paper makes following specific contributions:
\begin{enumerate}[\hspace{-3pt} 1.]
	\item We show how to extract \textit{all possible constraints}, for some constraint parameterization, consistent with a set of locally-optimal demonstrations, which we use to construct a belief over constraints.
	\item We provide a novel method for planning approximately-optimal open-loop trajectories between new start and goal states, which are safe under the constraint belief with a prescribed probability.
	\item We show how to use these open-loop trajectories to construct a closed-loop policy to adaptively satisfy the uncertain constraints, incorporating constraint information observed during execution.
	\item We theoretically analyze our algorithm, proving the completeness of constraint extraction for various constraint parameterizations and providing probabilistic safety guarantees for our planner.
	\item We evaluate our method by planning for a 7-DOF arm and a quadrotor with uncertain high-dimensional constraints, showing that our methods outperform baselines in efficiency and safety.
\end{enumerate}

\vspace{-6pt}
\section{Related Work}
\vspace{-2pt}

The constraint learning literature has shown how to learn global constraints shared across demonstrations \cite{ral, menner, corl}, geometric state space constraints \cite{dmitry, shah}, and local trajectory constraints \cite{lfdc4, lfdc1, vijayakumar, ParkNPRR19}. However, prior work makes the key assumption that the demonstrations are sufficiently informative for the learned constraint to be confidently used to plan new trajectories. As constraint learning is an ill-posed inverse problem, this is often untrue, and the robot should instead reason over the constraint uncertainty to complete the task. Two exceptions are \cite{borrelli}, which learns to satisfy linear constraints in iterative tasks -- whereas we learn to satisfy non-convex constraints to complete a task once -- and \cite{puns}, which plans open-loop to satisfy a finite set of possible temporal logic constraints -- while we plan closed-loop to adaptively satisfy an infinite set of possible low-level state space constraints.

Our work also relates to inverse optimal control \cite{irl_1, irl_2, birl}, imitation learning \cite{dagger}, and safety-focused variants \cite{safe_il1, safe_il2, safe_il3, safe_il4, safe_lfd1, safe_lfd2} that estimate the uncertainty in the learned policy. However, these approaches only use the uncertainty as a sign to switch to a safe backup policy \cite{safe_il1, safe_il3, safe_lfd2}, to evaluate the quality of a given policy \cite{safe_il2, safe_lfd1}, or to query the demonstrator for more data \cite{safe_il4}. In our setting, we must use the uncertainty from the demonstrations to compute the policy. Further, as we model the demonstrator's preferences and the task constraints separately, we can specifically reason over the uncertain constraints that can be learned in a self-supervised way, without the demonstrator.

Finally, our work relates to planning in uncertain environments. Some methods simplify the problem using Gaussian uncertainty \cite{axelrod} or constraint convexity \cite{vitus}; however, our constraint belief cannot be represented with these approximations. Other methods sample possible constraints and enforce them all, but this leads to plan infeasibility \cite{scenario}, while \cite{janson, richter} focus on high-quality short-range planning due to minimal knowledge of the global map. Our contributions are to show how demonstrations can reduce environment uncertainty, yielding better global plans, and to provide a method to plan chance-constrained trajectories for a class of uncertainty distributions with complex support. 

\vspace{-6pt}
\section{Preliminaries and Problem Setup}
\vspace{-2pt}

We consider demonstrations performed on systems $\state_{t+1} = f(\state_t, \control_t, t)$, $\state\in\statespace$, $\control\in\controlset$ completing tasks $\task\in\taskspace$, represented as constrained optimizations over state/control trajectories $\trajxu\doteq(\trajx,\traju)$:

\begin{problem}[Forward (demonstrator's) problem / ``task" $\task$]\label{prob:fwd_prob}
\normalfont
\vspace{2pt}
\begin{equation*}\label{eq:fwdprob}
	\begin{array}{>{\displaystyle}c >{\displaystyle}l >{\displaystyle}l}
				&\\[-10pt]
		\underset{\trajxu}{\text{minimize}} & \quad c_\task(\trajxu) &\\
		\text{subject to} & \quad \phi(\trajxu) \in \safeset(\theta) \subseteq \constraintspace & \Leftrightarrow\quad \mathbf{g}_{\neg k}(\trajxu, \theta) \le \mathbf{0}\\
		& \quad \bar\phi(\trajxu) \in \bar\safeset \subseteq \bar\constraintspace, \quad \phi_\task(\trajxu) \in \safeset_\task \subseteq \constraintspace_\task & \Leftrightarrow \quad \mathbf{h}_k(\trajxu) = \mathbf{0},\quad \mathbf{g}_{k}(\trajxu) \le \mathbf{0}
	\end{array}\hspace{-15pt}
\end{equation*}
\end{problem}

\noindent where $c_\task(\cdot)$ is task-dependent and $\phi(\cdot)$ maps from trajectories to constraint space $\constraintspace$; elements of $\constraintspace$ are denoted \textit{constraint states} $\cstate \in \constraintspace$. $\bar\phi(\cdot)$ and $\phi_\task(\cdot)$ map to constraint spaces $\bar\constraintspace$ and  $\constraintspace_\task$, containing a known shared safe set $\bar \safeset$ and task-dependent safe set $\safeset_\task$; we embed the dynamics in $\bar\safeset$ and start/goal constraints in $\safeset_\task$. We group the constraints of Prob. \ref{prob:fwd_prob} as (in)equality (ineq/eq) and (un)known ($\neg k/k$), where $\mathbf{h}_k(\trajxu)\hspace{-2pt} \in\hspace{-2pt} \mathbb{R}^{N_k^\eq}$, $\mathbf{g}_{k}(\trajxu)\hspace{-2pt} \in\hspace{-2pt} \mathbb{R}^{N_k^\ineq}$, and $\mathbf{g}_{\neg k}(\trajxu, \theta)\hspace{-2pt} \in\hspace{-2pt} \mathbb{R}^{N_{\neg k}^\ineq}$, and let $g(\cstate,\theta) \doteq \max_{i\in\{ 1, \ldots, N_{\neg k}^{\ineq}\}}\hspace{-4pt}\big(g_{i,\neg k}(\cstate, \theta)\big)$. Let the unknown safe and unsafe sets defined by parameter $\theta \in \Theta \subseteq \mathbb{R}^{d}$ be $\safeset(\theta)$ and $\unsafeset(\theta)$, respectively:

  \begin{equation}
  \safeset(\theta) \doteq \{\cstate \in \constraintspace\mid g(\cstate, \theta) \le 0\}
    \label{eq:safeset}
  \end{equation}
  
  \vspace{-5pt}\begin{equation}
  \unsafeset(\theta) \doteq \safeset(\theta)^c = \{\cstate \in \constraintspace\mid g(\cstate, \theta) > 0\}
    \label{eq:unsafeset}
  \end{equation}

We assume each state-control demonstration $\traj^\textrm{loc}$ approximately solves Prob. \ref{prob:fwd_prob} to local optimality, satisfying Prob. \ref{prob:fwd_prob}'s Karush-Kuhn-Tucker (KKT) conditions \cite{cvxbook} within a tolerance. Intuitively, this means $\traj^\textrm{loc}$ is feasible for Prob. \ref{prob:fwd_prob} (it remains within the safe set $\safeset(\theta)$ and satisfies the known constraints) and is within the neighborhood of a local optimum. With Lagrange multipliers $\lambda$, $\nu$, the relevant KKT conditions for the $j$th demonstration $\traj_{j}^\textrm{loc}$, denoted $\textrm{KKT}(\traj_{j}^\textrm{loc})$, are:

\vspace{-19pt}\begin{subequations}\label{eq:kkt}
\noindent
  \begin{equation}
  \mathbf{g}_{\neg k}(\traj_{j}^\textrm{loc}, \textcolor{red}{\theta}) \le \mathbf{0}\hspace{-3pt}
    \label{eq:kkt_primal3}
  \end{equation} \\\vspace{-15pt}
  \begin{equation}
  \textcolor{blue}{\boldsymbol{\lambda}_{k}^j}\odot\mathbf{g}_{k}(\traj_{j}^\textrm{loc}) = \mathbf{0}, \ \qquad \textcolor{blue}{\boldsymbol{\lambda}_{k}^j} \ge \mathbf{0}\hspace{-3pt}
    \label{eq:kkt_comp1}
  \end{equation} \\\vspace{-15pt}
  \begin{equation}
  \textcolor{blue}{\boldsymbol{\lambda}_{\neg k}^j}\odot\mathbf{g}_{\neg k}(\traj_{j}^\textrm{loc}, \textcolor{red}{\theta}) = \mathbf{0}, \ \qquad \textcolor{blue}{\boldsymbol{\lambda}_{\neg k}^j} \ge \mathbf{0}
    \label{eq:kkt_comp2}
  \end{equation}
\vspace{-10pt}
\begin{equation}
	\nabla_{\trajxu} c_\task(\traj_{j}^\textrm{loc}) + \textcolor{blue}{\boldsymbol{\lambda}_{k}^{j}}^\top \nabla_{\trajxu} \mathbf{g}_{k}(\traj_{j}^\textrm{loc})+\textcolor{blue}{\boldsymbol{\lambda}_{\neg k}^j}^{\hspace{-4pt}\top} \nabla_{\trajxu} \mathbf{g}_{\neg k}(\traj_{j}^\textrm{loc}, \textcolor{red}{\theta})+\textcolor{blue}{\boldsymbol{\nu}_{k}^j}^\top \nabla_{\trajxu} \mathbf{h}_{k}(\traj_{j}^\textrm{loc}) = \mathbf{0}
	\label{eq:kkt_stat}
\end{equation}
\end{subequations}
\vspace{-10pt}\begin{wrapfigure}{r}{.5\linewidth}\vspace{-6pt}
\begin{problem}[Inverse constraint learning problem]\label{prob:kkt_opt}
\normalfont\begin{equation*}\vspace{2pt}\label{eq:fwdprob}
	\hspace{-27pt}\begin{array}{>{\displaystyle}c >{\displaystyle}l >{\displaystyle}l}
				&\\[-17pt]
		\text{find} & \theta, \lag \doteq \{\boldsymbol{\lambda}_{k}^j, \boldsymbol{\lambda}_{\neg k}^j,\boldsymbol{\nu}_{k}^j\}_{j=1}^{\numsafe}\\
		\text{subject to} & \{\textrm{KKT}(\traj_{j}^\textrm{loc})\}_{j=1}^{\numsafe}\\[-2pt]
	\end{array}\hspace{-20pt}
\end{equation*}\vspace{-15pt}
\end{problem}\end{wrapfigure}
\vspace{-3pt}

\vspace{-5pt}\noindent where $\nabla_{\trajxu}(\cdot)$ differentiates with respect to a flattened $\trajxu$ and $\odot$ denotes elementwise product. \eqref{eq:kkt_primal3} enforces primal feasibility, \eqref{eq:kkt_comp1}-\eqref{eq:kkt_comp2} enforces complementary slackness, and stationarity \eqref{eq:kkt_stat} enforces the demonstration cannot be locally improved. As in \cite{ral}, we can solve Prob. \ref{prob:kkt_opt} to find constraints that make the demonstrations locally-optimal by finding a $\theta$ and Lagrange multipliers which are together consistent with the KKT conditions of the demonstrations. To handle approximate local-optimality in Prob. \ref{prob:kkt_opt}, we relax constraints \eqref{eq:kkt_comp1}-\eqref{eq:kkt_stat} to penalties. This framework can also learn unknown cost function parameters (c.f. App. \ref{app:shapes}). Let $\feas$ denote the feasible set of Prob. \ref{prob:kkt_opt}. Denote the projection of $\feas$ onto $\Theta$ (the set of all consistent constraints $\theta$) as $\feas_\theta$, and the projection of $\feastheta$ onto $\constraintset$ (the set of possibly-unsafe constraint states) as $\textrm{proj}_\constraintset(\feastheta)$:

\vspace{-5pt}
\begin{equation}
	\feas_\theta \doteq \{ \theta \mid \exists \lag: (\theta,\lag) \in \feas \}
\end{equation}
\vspace{-2pt}
	\begin{equation}\label{eq:projtheta}
		\textrm{proj}_\constraintset(\feastheta) \doteq \{\cstate\in\constraintset \mid \exists \theta \in \feastheta, g(\cstate, \theta) > 0\}
	\end{equation}\vspace{3pt}

While Prob. \ref{prob:kkt_opt} returns \textit{one possible} $\theta$, there can be an \textit{infinite set} of possible $\theta$: $\feastheta$. Thus, $\feastheta$ represents the constraint uncertainty. We use $\feastheta$ to build a constraint belief (Sec. \ref{sec:extraction}) and use $\textrm{proj}_\constraintset(\feastheta)$ for planning (Sec. \ref{sec:planning}). Finally, let the set of learned guaranteed-safe/unsafe constraint states be $\guarsafe$ and $\guarunsafe$, where $\cstate$ is learned guaranteed (un)safe if it is marked (un)safe for all $\theta \in \feastheta$ (c.f. Fig. \ref{fig:projection_extraction}):

	\begin{equation}\label{eq:guarsafe}
		\guarsafe \doteq \bigcap_{\theta \in \feas_\theta} \{ \cstate\ |\ g(\cstate, \theta) \le 0 \}
	\end{equation}

	\begin{equation}\label{eq:guarunsafe}
		\guarunsafe \doteq \bigcap_{\theta \in \feas_\theta} \{ \cstate\ |\ g(\cstate, \theta) > 0 \}
	\end{equation}

Previous work \cite{ral} plans guaranteed-safe trajectories by enforcing that they always remain in $\guarsafe$, but $\guarsafe$ can be tiny or disconnected (Fig. \ref{fig:projection_extraction}), making planning infeasible. In this work, we allow plans to pass through possibly-unsafe space $\textrm{proj}_\constraintset(\feastheta)$, but seek to minimize constraint violations.

\textbf{Problem statement}: We are given $\numsafe$ demonstrations $\{\traj_{j}^\textrm{loc} \}_{j=1}^{\numsafe}$, known shared and task-dependent safe sets $\bar\safeset$ / $\safeset_\task$, and a prior $p(\theta)$ over the unknown constraint. Our goals are: 1) recover the set of all constraint parameters $\feastheta \subseteq \Theta$ consistent with the demonstrations to obtain a constraint belief $b_\textrm{dem}(\theta) \doteq p(\theta \mid \{\traj_{j}^\textrm{loc} \}_{j=1}^{\numsafe}) \in \mathcal{P}(\Theta)$, and 2) compute a policy $\policy(\cdot, \cdot, \cdot): \mathcal{P}(\Theta) \times \statespace \times (\mathcal{O})^* \rightarrow \controlset$, which takes a prior, start state $\state_0$, and a sequence of constraint observations, and returns a control input $\control$, and completes a task (in this paper, we consider a task as reaching a goal $\state_g$ from $\state_0$) while minimizing one of two objectives. In the first variant, denoted Prob. MCV, we want to reach the goal with the \underline{m}inimum number of expected \underline{c}onstraint \underline{v}iolations. In the second variant, denoted Prob. MEC, we want to \underline{m}inimize the \underline{e}xpected \underline{c}ost of some general objective function.

\begin{figure}[!htb]
        \centering\vspace{-2pt}
        \includegraphics[width=\linewidth]{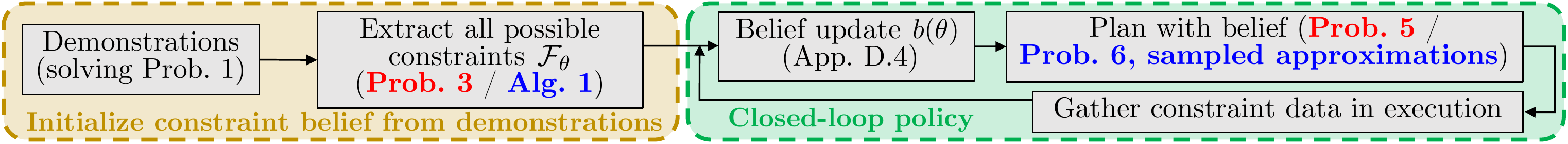}\vspace{-3pt}
        \caption{Overall method flow for adaptive planning from demonstrations. We refer to both the ideal (red), but intractable, subproblems, as well as the tractable (blue) variants of those subproblems.\vspace{-0pt}}
        \label{fig:flow}
\end{figure}
\textbf{Method overview}: To prime the reader, we outline two variants of our method in Fig. \ref{fig:flow}: 1) an ideal variant that requires the solution of intractable optimizations, and 2) tractable variants which approximate the idealized problems or exploit simplifying problem structure. For closed-loop planning, both variants compute the constraint belief (Sec. \ref{sec:extraction}), then iteratively plan with the belief, update the belief with constraint data measured in execution, and replan with the updated belief (Sec. \ref{sec:planning}).

\vspace{-3pt}
\section{Obtaining a belief over constraints}\label{sec:extraction}
\vspace{-2pt}

\begin{wrapfigure}{r}{0.41\linewidth}\vspace{-27pt}
\begin{figure}[H]
        \centering\vspace{-20pt}
        \includegraphics[width=\linewidth]{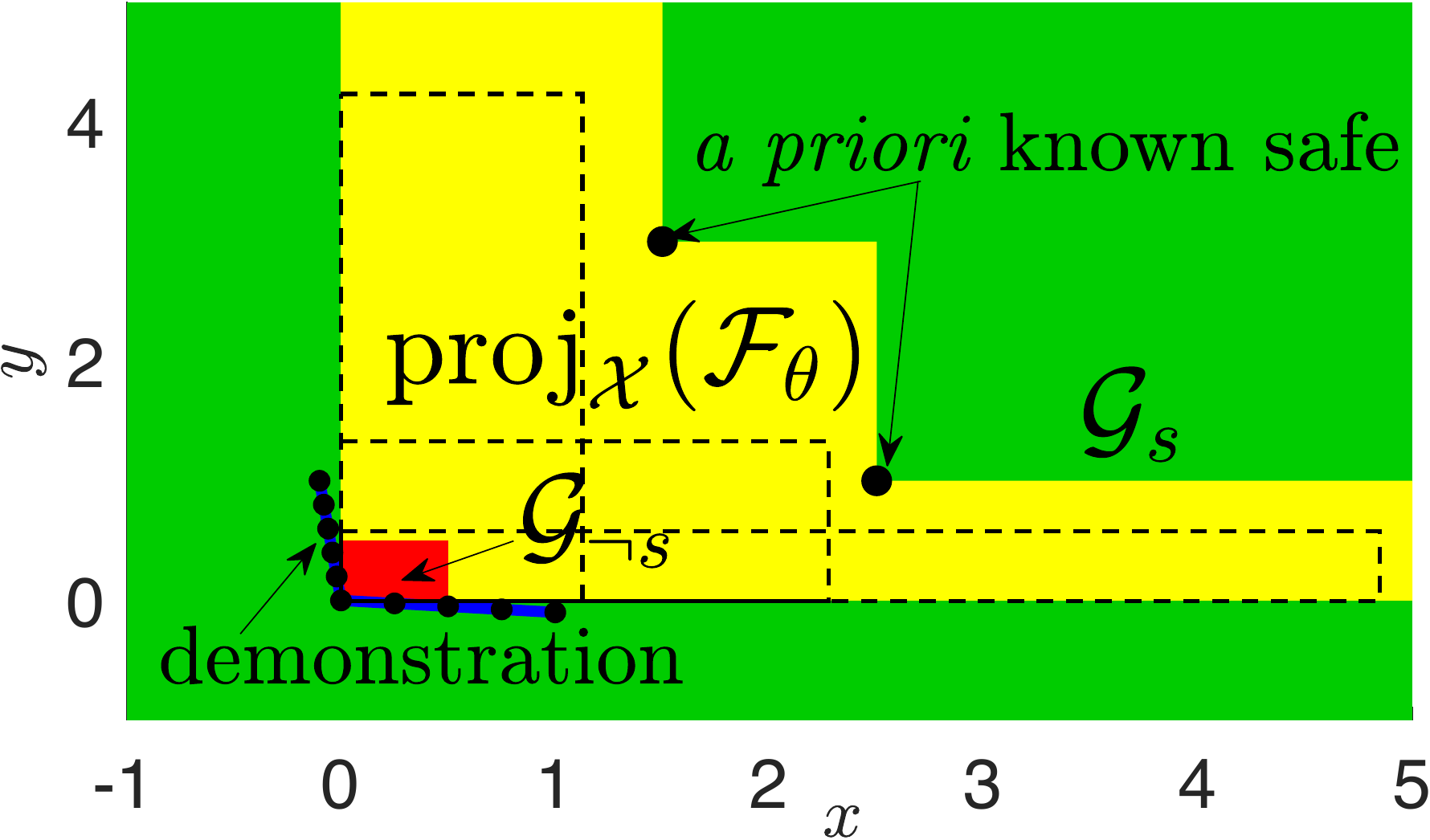}\vspace{5pt}
        \caption{$\feastheta$ for a one-box parameterization of $\unsafeset(\theta)$, induced by a demo. and two safe states, projected onto $\statespace$. With the data, the upper $x$ / $y$ bounds $\bar x(\theta)$ / $\bar y(\theta)$ remain uncertain. Also: some possible $\unsafeset(\theta)$ (dotted).}
        \label{fig:projection_extraction}\vspace{-11pt}
\end{figure}\vspace{-35pt}

\end{wrapfigure}

Extracting $\feastheta$ is crucial for obtaining an accurate belief over constraints. In this section, we show how to use robust optimization to obtain $\feas_\theta$ for some constraint parameterizations. We can robustify Prob. \ref{prob:kkt_opt}, where $\theta$ is considered as an uncertain variable in uncertainty set $\feasrobust \subseteq \Theta$:

\vspace{2pt}\begin{problem}[Inverse constraint learning, robustified in $\theta$]\label{prob:kkt_opt_robust}
\normalfont\begin{equation*}\vspace{2pt}\label{eq:fwdprob}
	\hspace{-27pt}\begin{array}{>{\displaystyle}c >{\displaystyle}l >{\displaystyle}l}
				&\\[-14pt]
		\underset{\feasrobust}{\text{sup}} & \textrm{Vol}(\feasrobust) \\[-4pt]
		\text{s.t.} & \forall \theta \in \hat\feas_\theta,\ \ \exists \{\boldsymbol{\lambda}_{k}^j, \boldsymbol{\lambda}_{\neg k}^j,\boldsymbol{\nu}_{k}^j \mid \textrm{KKT}(\traj_{j}^\textrm{loc})\}_{j=1}^{\numsafe}\\[-4pt]
	\end{array}\hspace{-20pt}
\end{equation*}
\end{problem}\vspace{-2pt}

\vspace{10pt}and search for the largest set $\feasrobust \subseteq \Theta$ where each $\theta \in \feasrobust$ satisfies KKT; the optimizer of Prob. \ref{prob:kkt_opt_robust} is $\feastheta$. However, Prob. \ref{prob:kkt_opt_robust} is intractable due to 1) the optimization over arbitrarily-shaped sets $\feasrobust$, 2) measuring the volume of such sets, and 3) the existential quantifiers $\exists$, implying we may need to find different Lagrange multipliers for each $\theta \in \feasrobust$. We address these challenges in the following.

\subsection{Obtaining the set of demonstration-consistent constraints $\feastheta$}\label{sec:extraction_uob}

We assume the unknown constraint $\unsafeset(\theta)$ can be represented as a union of boxes in constraint space $\{\cstate \mid \bigcup_i [I, -I]^\top \cstate \le \theta_i\}$. This assumption is reasonable as any shape can be represented by unioning enough boxes \cite{tao}, though this can be inefficient (thus, we relax the assumption in App. \ref{app:extraction_zon}). In App. \ref{sec:unionsofoffset}, we prove that if $\unsafeset(\theta)$ can be described as a union of boxes, so can $\feas_\theta$. By exploiting this structure, we develop a tractable variant of Prob. \ref{prob:kkt_opt_robust} using robust linear programming \cite{Ben-TalGN09}.

\vspace{4pt}We address the challenge of set optimization by optimizing over only boxes. Using the identity $\sup_{\Vert u \Vert_\infty \le 1} a^\top u = \Vert a \Vert_1$, a linear constraint $a^\top (x + s\odot u) \le b$ involving uncertain variable $u: \Vert u \Vert_\infty \le 1$, can be equivalently written without $u$ as $a^\top x + \Vert a \odot s \Vert_1 \le b$, where $s \in \mathbb{R}_{\ge 0}^{d}$ scales the uncertainty. We can use this idea to enforce that the KKT conditions robustly hold everywhere in some box $\theta + s \odot u$, where $\Vert u\Vert_\infty \le 1$. Concretely, we can replace \eqref{eq:kkt_primal3} with $\mathbf{g}_{\neg k}(\traj_{j}^\textrm{loc}, \textcolor{red}{\theta + s \odot u}) \le \mathbf{0}$, \eqref{eq:kkt_comp2} with $\textcolor{blue}{\boldsymbol{\lambda}_{\neg k}^j}\odot\mathbf{g}_{\neg k}(\traj_{j}^\textrm{loc}, \textcolor{red}{\theta + s \odot u}) = \mathbf{0}$, and \eqref{eq:kkt_stat} with $\nabla_{\trajxu} c(\traj_{j}^\textrm{loc}) + \textcolor{blue}{\boldsymbol{\lambda}_{k}^{j}}^\top \nabla_{\trajxu} \mathbf{g}_{k}(\traj_{j}^\textrm{loc}) + \textcolor{blue}{\boldsymbol{\lambda}_{\neg k}^j}^{\hspace{-4pt}\top} \nabla_{\trajxu} \mathbf{g}_{\neg k}(\traj_{j}^\textrm{loc}, \textcolor{red}{\theta + s \odot u}) + \textcolor{blue}{\boldsymbol{\nu}_{k}^j}^\top \nabla_{\trajxu} \mathbf{h}_{k}(\traj_{j}^\textrm{loc}) = \mathbf{0}$, and eliminate $u$ with the identity. We denote \eqref{eq:kkt_comp1} and the robustified \eqref{eq:kkt_primal3}, \eqref{eq:kkt_comp2}, and \eqref{eq:kkt_stat} together as $\textrm{KKT}_\textrm{rob}^\textrm{box}(\demj)$, which are representable in a mixed integer linear program (MILP) (we need binary variables to enforce the robustified \eqref{eq:kkt_comp2}).

\begin{wrapfigure}{r}{0.5\linewidth}\vspace{-8pt}
\begin{figure}[H]
        \centering\vspace{-18pt}
        \includegraphics[width=0.94\linewidth]{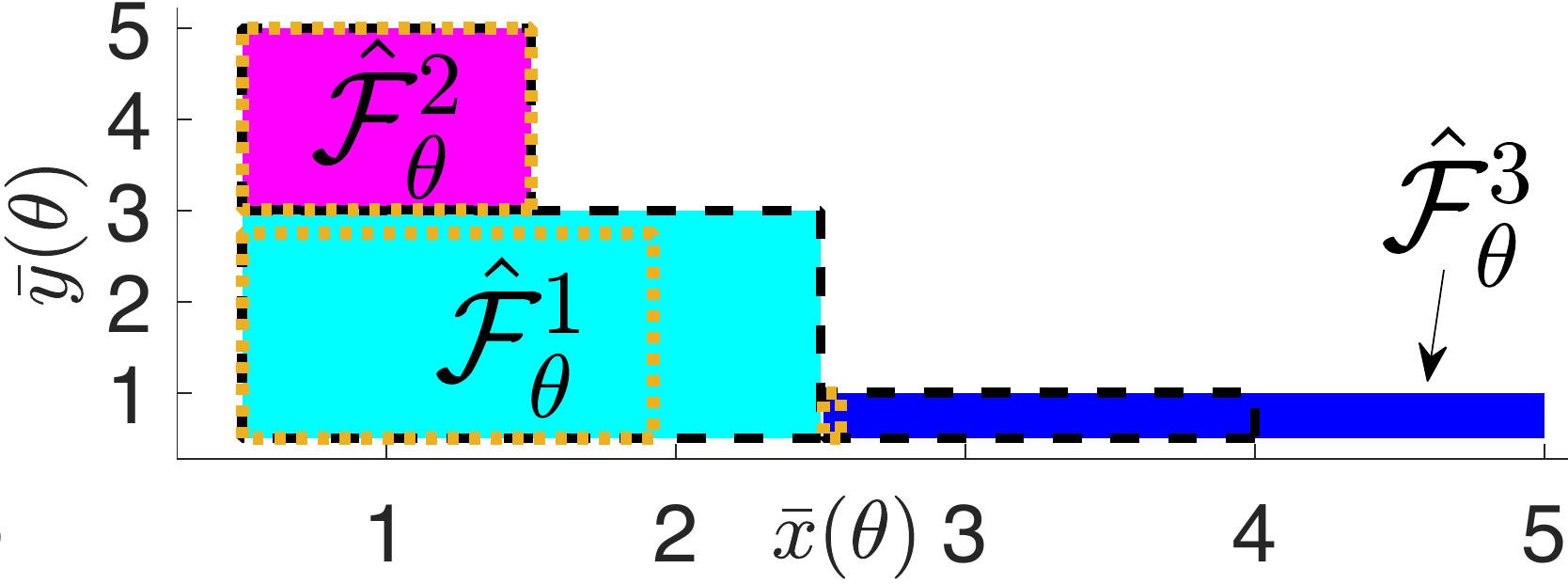}\vspace{-5pt}
        \caption{Extracting $\feastheta$ via Alg. \ref{alg:extraction}: requires 3 iterations. Overlaid: $\mathcal{B}_i^{\varepsilon_\textrm{min}}$ (black dotted) and $\mathcal{B}_i^{R}$ (orange dotted), $i=1,...,3$, optimized by solving Probs. \ref{prob:cc_riemann}-$\varepsilon_\textrm{min}$ and \ref{prob:cc_riemann}-R (plans in Fig. \ref{fig:plan}). \vspace{-35pt}}
        \label{fig:extraction}
\end{figure}\vspace{-0pt}
\end{wrapfigure}

\vspace{4pt}The box representation of $\feasrobust$ simplifies volume optimization. Since $s$ scales the uncertainty (and thus the volume of $\feasrobust$), we can satisfy $a^\top (x + s\odot u) \le b$ with ``maximum robustness" by jointly finding $x$ and $s$ to maximize the volume of $\feasrobust$, $\prod_i s_i$, where $s_i$ is the $i$th entry of $s$. While $\prod_i s_i$ is non-convex in $s$, its geometric mean, $(\prod_i s_i)^{1/d}$, is conic-representable \cite{SOCP}. It is also a monotonic transform of the volume, and thus an exact surrogate for volume maximization. 

\vspace{4pt}Finally, we can ignore the existential quantifiers in this case: \eqref{eq:kkt_stat} does not involve $\theta$ (as $\trajxu$ does not multiply $\theta$), \eqref{eq:kkt_comp2} implies that $\lambda_{\neg k, i}^j = 0$ for any coordinates $i$ where $g_{\neg k, i}(\demj, \theta)$ varies (hence one value, $\lambda=0$, suffices), and \eqref{eq:kkt_primal3} does not involve $\lag$. Thus, a single set of multipliers suffices, and we find the largest box-shaped $\feasrobust$ via Prob. \ref{prob:kkt_opt_robust_boxes}, a mixed integer second order cone program (MISOCP).

\vspace{-3pt}\begin{wrapfigure}{l}{.4\linewidth}\vspace{-8pt}
\begin{problem}[Box robustification]\label{prob:kkt_opt_robust_boxes}
\normalfont\begin{equation*}\vspace{2pt}\label{eq:fwdprob}
	\hspace{-0pt}\begin{array}{>{\displaystyle}c >{\displaystyle}l >{\displaystyle}l}
				&\\[-15pt]
		\underset{s, \theta, \lag }{\text{maximize}} & \big(\textstyle\prod_i s_i\big)^{1/d} \\
		\text{subject to} & \{\textrm{KKT}_\textrm{rob}^\textrm{box}(\traj_{j}^\textrm{loc})\}_{j=1}^{\numsafe}\\[1pt]
	\end{array}\hspace{-20pt}
\end{equation*}
\end{problem}\vspace{-20pt}
\end{wrapfigure}

\vspace{3pt}Solving Prob. \ref{prob:kkt_opt_robust_boxes} returns the largest box contained within $\feastheta$: $\feasrobust = \{\theta' \mid \bigwedge_{i=1}^{d} |\theta_i' - \theta_i| \le s_i \} \subseteq \feastheta$. As $\feastheta$ is a union of boxes, we can extract $\feastheta$ in its entirety by solving Prob. \ref{prob:kkt_opt_robust_boxes}, removing the extracted $\feasrobust$ from its feasible set (done with binary variables), and re-solving Prob. \ref{prob:kkt_opt_robust_boxes} with the modified feasible set until it becomes infeasible (Alg. \ref{alg:extraction}, Fig. \ref{fig:extraction}). Concretely, $\feastheta = \bigcup_{i=1}^{N_\textrm{infeas}} \feasrobust^i$, where $\feasrobust^i$ is the box returned at the $i$th iteration and $N_\textrm{infeas}$ is the iteration when infeasibility is reached. We can also prove some theoretical guarantees on Alg. \ref{alg:extraction} (see App. \ref{app:theory} for proofs).

\vspace{2pt}
\begin{theorem}
	If Alg. \ref{alg:extraction} terminates for any parameterization, its output is guaranteed to cover $\feastheta$.
\end{theorem}\vspace{-2pt}
\begin{theorem}
	Alg. \ref{alg:extraction} is guaranteed to terminate in finite time for union-of-boxes parameterizations.
\end{theorem}

\vspace{-5pt}\begin{algorithm}[H]
\SetAlgoLined
$i = 0$; \While{Prob. \ref{prob:kkt_opt_robust_boxes} feasible}{
\hspace{-5pt}$i \leftarrow i + 1$;\ \ $\mathcal{\hat F}_\theta^i \leftarrow$ Prob. \ref{prob:kkt_opt_robust_boxes}($\{\mathcal{\hat F}_\theta^j\}_{j=1}^{i-1}$);\\
\hspace{-5pt}remove $\mathcal{\hat F}_\theta^i$ from Prob. \ref{prob:kkt_opt_robust_boxes}'s feasible set;\\}
return $\bigcup_i \mathcal{\hat F}_\theta^i$\caption{Iterative $\feastheta$ extraction}\label{alg:extraction}
\end{algorithm}

In closing, we refer to App. \ref{app:extraction_zon}, where we modify Alg. \ref{alg:extraction} to more efficiently extract $\feastheta$ for other constraint parameterizations by covering $\feastheta$ with zonotopes instead of boxes.

\vspace{-5pt}
\subsection{Obtaining the constraint belief $b(\theta)$}
\vspace{-2pt}

To perform a Bayesian update of $p(\theta)$, conditioning on the extracted $\feastheta$, we assume that a demonstration is equally likely to have been generated in response to any $\theta$ for which it is locally-optimal:

\begin{equation}
	p(\{\demj\}_{j=1}^{N_\textrm{dem}} \mid \theta) \propto \begin{cases}
                                   1 & \text{if $\{\textrm{KKT}(\demj, \theta)\}_{j=1}^{N_\textrm{dem}}$ satisfied } \\
                                   0 & \text{else}
 								 \end{cases}
\end{equation}

Then, a Bayesian update incorporating the demonstrations amounts to removing all probability mass from KKT-inconsistent $\theta$ and renormalizing the probabilities for the KKT-consistent $\theta$: 

\begin{equation}
		b_\dem(\theta) \doteq p(\theta \mid \{\demj\}_{j=1}^{N_\textrm{dem}}) = \frac{p(\{\demj\}_{j=1}^{N_\textrm{dem}} \mid \theta)p(\theta)}{\int_{\Theta} p(\{\demj\}_{j=1}^{N_\textrm{dem}} \mid \theta) p(\theta) d\theta} = \begin{cases} \frac{p(\theta)}{\int_{\feastheta} p(\theta) d\theta} & \text{if }\theta \in \feastheta \\ 0 & \textrm{else} \end{cases}
\end{equation}

Finally, we note that this approach is also compatible with uninformative priors (i.e. if no demonstrations are provided) by using the initial prior as the belief: $b(\theta) = p(\theta)$.

\vspace{-6pt}
\section{Policies for adaptive constraint satisfaction}\label{sec:planning}
\vspace{-4pt}

We describe how to use the belief over infinite constraints $b_\dem(\theta)$ to plan open-loop trajectories with exact safety probability guarantees (Sec. \ref{sec:plan_open_loop}), how to plan with more complex constraints with samples from $b_\dem(\theta)$ (Sec. \ref{sec:plan_open_loop_samp}), how $b_\dem(\theta)$ can be updated to use constraint data sensed in execution (Sec. \ref{sec:plan_feasupdate}), and how the open-loop plans can be used in a closed-loop policy (Sec. \ref{sec:plan_policy}).

\vspace{-4pt}
\subsection{Planning open-loop trajectories with an \textit{infinite set} of possible constraints}\label{sec:plan_open_loop}\vspace{1pt}

\begin{wrapfigure}{l}{0.38\textwidth}
\vspace{-18pt}\begin{problem}{Chance-constrained plan}\label{prob:cc}\normalfont
\begin{subequations}
	\begin{align}
    \min_{\trajxu }\quad & c_{\task}(\trajxu) \label{eq:cost_cc}\\
    \textrm{s.t.}\quad  & \bar\phi(\trajxu) \in \bar\safeset \subseteq \bar\constraintspace\\
	& \phi_{\task}(\trajxu) \in \safeset_{\task} \subseteq \constraintspace_{\task} \\
    & \textrm{Pr}(\trajxu \textrm{ safe}) \ge 1-\varepsilon\label{eq:prob_safety}
\end{align}
\end{subequations}
\end{problem}\vspace{-20pt}
\end{wrapfigure} We wish to solve Prob. \ref{prob:cc} for convex \eqref{eq:cost_cc}, which seeks to complete a task while ensuring the plan is safe with probability at least $1-\varepsilon$ under the belief $b_\dem(\theta)$, that is, $\pr(\trajxu\safe) = \int_{\Theta_s} b_\dem(\theta) d\theta$, where $\Theta_s = \{\theta \mid \phi(\trajxu) \in\safeset(\theta)\} \subseteq \feastheta$ is the set of constraints that $\trajxu$ satisfies. Here, the safety threshold $\varepsilon \in [0, 1]$ may be predetermined, or if we wish to plan the safest possible trajectory, we can find the smallest $\varepsilon$ for which Prob. \ref{prob:cc} is feasible; denote this variant as Prob. \ref{prob:cc}-$\varepsilon_\textrm{min}$. Intuitively, Prob. \ref{prob:cc} seeks to solve a chance-constrained variant of Prob. \ref{prob:fwd_prob} for a novel task $\task$, where the uncertain constraints must be satisfied with a sufficiently high probability. Prob. \ref{prob:cc} is challenging due to \eqref{eq:prob_safety}, as evaluating this probability requires integrating high-dimensional parameters $\theta$ over a possibly arbitrarily-shaped $\Theta_s$; hence, \eqref{eq:prob_safety} is intractable to enforce exactly for arbitrary distributions and $\Theta_s$. We show that by sacrificing global optimality, it is tractable to enforce \eqref{eq:prob_safety} exactly for simple priors $p(\theta)$ by assuming a simple shape for $\Theta_s$.

\begin{wrapfigure}{r}{0.55\textwidth}\vspace{-10pt}
\begin{problem}{Riemann-sum chance-constrained plan}\normalfont\label{prob:cc_riemann}
\vspace{-0pt}\begin{subequations}\small
	\begin{align}
    \hspace{-5pt}\min_{\trajxu, \mathcal{B}_i, t_i }\ & c_\task(\trajxu)\label{eq:cost_riemann} \\[-4pt]
    \textrm{s.t.}\quad\  & \bar\phi(\trajxu) \in \bar\safeset \subseteq \bar\constraintspace,\ \phi_\task(\trajxu) \in \safeset_\task \subseteq \constraintspace_\task \label{eq:known_riemann} \\
	& \trajxu \in \safeset(\theta),\ \forall \theta \in \mathcal{B}_1, \ldots, \mathcal{B}_{N_\textrm{box}} \label{eq:int_riemann1} \\
    & \mathcal{B}_i \cap \mathcal{B}_j = \emptyset, \ i \ne j,\quad \mathcal{B}_i \subseteq \feastheta, \forall i \label{eq:int_riemann2}\\
    & 0 \le t_i \le (\textstyle\prod_i b_i^\textrm{scale})^{1/d}, \ i = 1,..., N_\textrm{box}\label{eq:int_riemann3}\\
    & \textstyle\sum_i t_i^{d} \ge (1-\varepsilon)\textrm{Vol}(\feastheta)\label{eq:cc_norm}
\end{align}
\end{subequations}
\end{problem}\vspace{-10pt}
\end{wrapfigure}

Our solution, Prob. \ref{prob:cc_riemann}, optimizes over subsets $\Theta_s$ that can be represented as a union of boxes $\Theta_s = \bigcup_{i=1}^{N_\textrm{boxes}} \mathcal{B}_i$, $\mathcal{B}_i \subseteq \feastheta$, for all $i$ (c.f. Sec. \ref{sec:extraction_uob}). Each box is parameterized with a center $b_i^\textrm{cen} \in \mathbb{R}^{d}$ and scalings $b_i^\textrm{scale} \in \mathbb{R}_+^{d}$: $\mathcal{B}_i = \{b_i^\textrm{cen} + b_i^\textrm{scale} \odot u \mid u \in [-1, 1]^{d}\}$. \eqref{eq:int_riemann1}-\eqref{eq:cc_norm} implement this box-limited chance constraint (see detailed explanations for each constraint in App. \ref{app:overview_glossary}). We restrict focus to priors $p(\theta)$ that can be integrated over boxes in closed form, and for which a monotonic transformation of the resulting integral is concave in $b_i^\textrm{cen}$ and $b_i^\textrm{scale}$. While the concavity assumption is satisfied by the broad class of log-concave distributions \cite{logconcave}, the closed-form integral is more restrictive. In this paper, we focus only on a uniform $p(\theta)$ (see App. \ref{app:planning_priors} for extensions to other distributions); note that this does \textit{not} imply a uniform probability of safety over the constraint space $\constraintset$. Intuitively, Prob. \ref{prob:cc_riemann} performs a box-limited Riemann sum integration over the constraint belief. Each box $\mathcal{B}_i$ represents a subset of $\feastheta$ over which the probability is integrated (c.f. Fig. \ref{fig:extraction}). For piecewise affine (PWA) dynamics, Prob. \ref{prob:cc_riemann} can be written as an MISOCP, except for \eqref{eq:cc_norm} which renders Prob. \ref{prob:cc_riemann} an MIBLP: solvable with \cite{gurobi}, but possibly slow. We can replace \eqref{eq:cc_norm} with a linear surrogate $\sum_i t_i \ge (1-\varepsilon)\textrm{Vol}(\feastheta)$, but this can still be slow if $N_\textrm{box}$ is large. We discuss efficient reformulations of Prob. \ref{prob:cc_riemann} in App. \ref{app:planning}. Overall, we have this result (proof in App. \ref{app:theory}):
 
\vspace{5pt} \begin{theorem}
	A solution to Prob. \ref{prob:cc_riemann} is a guaranteed feasible, possibly suboptimal solution to Prob. \ref{prob:cc}.
\end{theorem}

\begin{wrapfigure}{l}{0.5\linewidth}\vspace{-17pt}
\begin{figure}[H]
        \vspace{-10pt}\centering
        \includegraphics[width=0.99\linewidth]{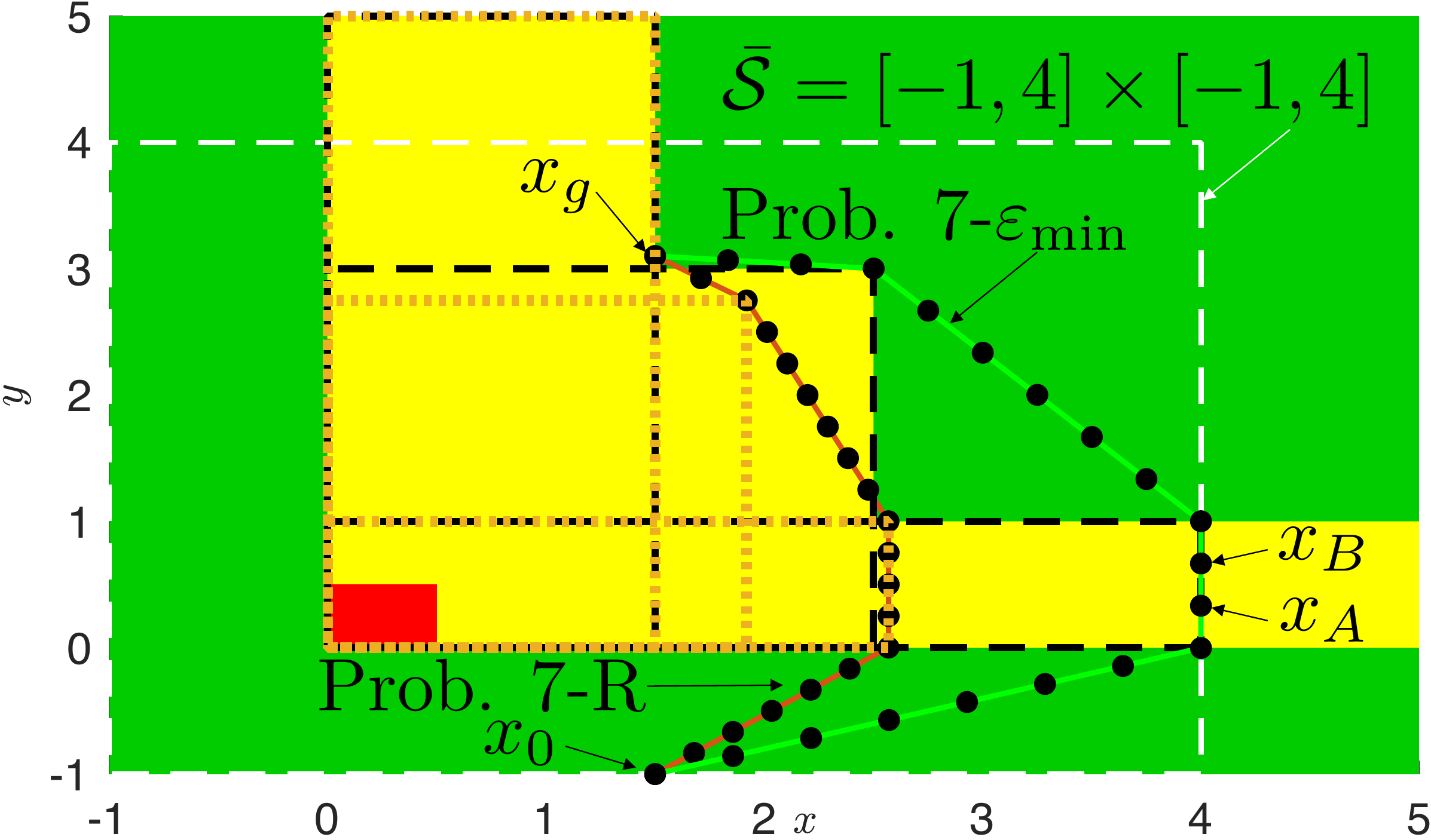}\vspace{-9pt}
        \caption{Generated plans for $\mathcal{B}_i$ in Fig. \ref{fig:extraction}. $\textrm{proj}_\statespace(\mathcal{B}_i)$ are overlaid (with matched color).}
        \label{fig:plan}
\end{figure}\vspace{-30pt}
\end{wrapfigure}

Instead of the chance-constrained formulation, we can directly trade off the cost and the safety probability in the objective; i.e. in Probs. \ref{prob:cc}; \ref{prob:cc_riemann}, change \eqref{eq:cost_cc}; \eqref{eq:cost_riemann} to the \underline{r}atio $c_\task(\trajxu)/\pr(\trajxu\safe)$; $c_\task(\trajxu)/\sum_i t_i^d$ and remove \eqref{eq:prob_safety}; \eqref{eq:cc_norm}. Denote these variants Prob. \ref{prob:cc}-R; Prob. \ref{prob:cc_riemann}-R. Note that after linearizing $\sum_i t_i^d$, the objective is quasi-convex, so we can rewrite Prob. \ref{prob:cc_riemann}-R as a feasibility problem with a new optimization constraint $c_\task(\trajxu) - \alpha \pr(\trajxu\safe) \le 0$, and do a line search on $\alpha$, solving Prob. \ref{prob:cc_riemann}-R as a sequence of MISOCPs.

\vspace{5pt}
\subsection{Planning open-loop trajectories with a \textit{finite set} of sampled possible constraints}\label{sec:plan_open_loop_samp}

For complex constraints arising from nonlinear dynamics or constraint parameterizations, Prob. \ref{prob:cc_riemann} is hard to solve as it involves both integer variables and nonlinearities. While sometimes we can plan for a PWA model that roughly captures the nonlinear dynamics (i.e. modeling a quadrotor as a double integrator), we generally use sampled approximations of Prob. \ref{prob:cc_riemann} when it cannot be written as a mixed integer convex program (MICP), in particular, Minimum Constraint Removal (MCR) \cite{MCR} and the Blindfolded Traveler's Problem (BTP) \cite{BTP}, briefly described here (c.f. App. \ref{app:planning} for details):

\begin{enumerate}[I.]
	\item MCR takes a \textit{finite} set of constraints and incrementally constructs a roadmap to connect a start and goal state while violating the minimum number of constraints. MCR can be used to approximate Prob. \ref{prob:cc_riemann}-$\varepsilon_\textrm{min}$ by sampling a finite set of constraints $\{\theta_i\sim b(\theta)\}_{i=1}^{N_\textrm{sam}}$ as input to MCR.
	\item \vspace{-2pt}BTP is a roadmap-based planner for additive cost functions $c_\task(\trajxu)$ which takes as input a state space graph $(V,E)$, where executing edge $e \in E$ costs $c(e)$ with probability $p(e)$ of being safe. To use BTP, we sample constraints $\{\theta_i\sim b(\theta)\}_{i=1}^{N_\textrm{sam}}$, and use them to approximate $p(e)$; we plan on the graph by running A* with modified edge costs $\sum_e c_\task(e) - \beta p(e \safe)$, for some weight $\beta$.
\end{enumerate}

We note that while the sampled approximations can be more flexible than Prob. \ref{prob:cc_riemann}, they are not guaranteed to return a feasible solution to Prob. \ref{prob:cc}, as it depends on the constraints that are sampled.

\vspace{5pt}
\subsection{Updates to $b(\theta)$ in online execution}\label{sec:plan_feasupdate}
\vspace{5pt}

Our framework can also incorporate uncertain information about the true constraint sensed in execution by computing a belief update. Suppose that we are given $\collset_s \doteq \{\collset_s^i\}_{i=1}^{N_s}$ and $\collset_{\neg s} \doteq \{\collset_{\neg s}^i\}_{i=1}^{N_{\neg s}}$ as a set of \textit{sets of possibly safe/unsafe states}, respectively, where each $\collset_{s}^i$ / $\collset_{\neg s}^i$ denotes a finite set where at least one state is safe/unsafe. Let the set of all constraints consistent with $\collset_s$, $\collset_{\neg s}$ be $\feasthetasun \doteq \{\theta \in \feastheta \mid \bigwedge_{i=1}^{N_s} (\exists \cstate \in \collset_s^i, g(\cstate, \theta) \le 0) \wedge \bigwedge_{i=1}^{N_{\neg s}} (\exists \cstate \in \collset_{\neg s}^i, g(\cstate, \theta) > 0)\}$. We compute $\feasthetasun$ iteratively with a variant of Alg. \ref{alg:extraction} (see App. \ref{app:planning_update} for details). Finally, we perform the update: 

\vspace{-7pt}
\begin{equation}
	b_\textrm{ex}(\theta) \doteq p(\theta\mid \{\demj\}_{j=1}^{N_\dem}, \constraintset_{s}, \constraintset_{\neg s}) = \begin{cases} \frac{p(\theta)}{\int_{\feasthetasun} p(\theta) d\theta} & \text{if }\theta \in \feasthetasun \\ 0 & \textrm{else} \end{cases}
\end{equation}\vspace{5pt}

\begin{wrapfigure}{r}{0.52\linewidth}\vspace{-25pt}
        \centering
        \includegraphics[width=0.6\linewidth]{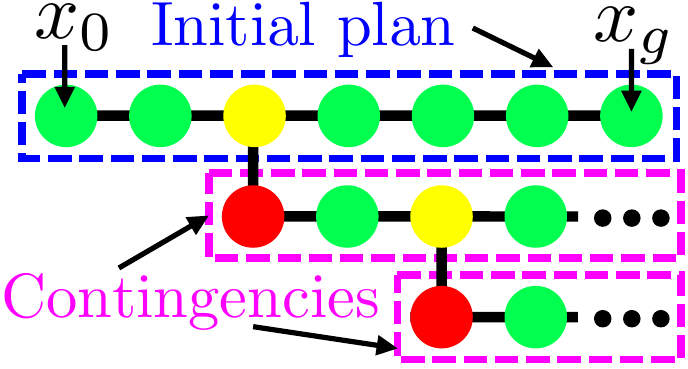}\vspace{-3pt}
        \caption{{Policy tree: initial plan and contingencies, rooted at possibly unsafe states. Green / red / yellow states in $\guarsafe$ / $\guarunsafe$ / $\textrm{proj}_\constraintset(\feastheta)$.}\vspace{-31pt}}
        \label{fig:branching}
\end{wrapfigure}
This setup can handle data from many different constraint sensing modalities, like direct, exact sensing (from a bump sensor), long-range measurements (from bounded-range LiDAR), or uncertain contact measurements \cite{BTP} where collision cannot be exactly localized on the robot volume (see App. \ref{app:planning_update} for more details).

\vspace{5pt}
\subsection{Closed-loop policies for adaptive constraint satisfaction}\label{sec:plan_policy}

Finally, we compute closed-loop policies for Probs. MCV and MEC. Our strategy is simple: for Probs. MCV and Prob. MEC, solve Prob. \ref{prob:cc_riemann}-$\varepsilon_\textrm{min}$ and Prob. \ref{prob:cc_riemann}-R, respectively, in a receding horizon fashion, i.e at each time-step, we update $b(\theta)$ with the new constraint information, re-solve the optimization, and switch to the new solution if the previous plan is suboptimal/unsafe. This policy takes $p(\theta)$ and a sequence of observations $\mathcal{O} = (\collset_s, \collset_{\neg s})$ to estimate a belief, and uses the belief and current state to output $\control$. These strategies are motivated by results in sequential decision making that provide approximation guarantees for greedy policies proposing solutions that minimize the ratio of cost to safety probability at each iteration \cite{Dor}; Probs. \ref{prob:cc_riemann}-$\varepsilon_\textrm{min}$, \ref{prob:cc_riemann}-R both do this.

Note that as our policy executes a plan until sensing implies it is suboptimal or unsafe, upon which it switches, it can be represented as a tree, where contingency plans are rooted only at states on the current plan where switches can occur. See Fig. \ref{fig:branching} for the simple case where the policy only switches upon learning the current plan is unsafe; here, the branching is sparse, occurring only at possibly unsafe points (yellow) on the initial and replanned trajectories. In these cases, we can exploit the sparsity to avoid solving Prob. \ref{prob:cc_riemann} at runtime by precomputing the contingency plans, facilitating real-time policy execution. For tractability, the precomputation assumes no unmodeled obstacles appear at runtime (this eliminates the sparse branching, as it makes each state on the plan possibly unsafe), discretizes the set of possible continuous sensing measurements (if not, we could need to compute contingencies for an infinite set of possible measurements), and terminates branching at a finite tree depth (as planning may be infeasible for a worst-case constraint). If the assumptions do not hold (as in some of our results), we can always compute new plans online, though it can be slow. 

As an example, consider computing contingencies for the green plan in Fig. \ref{fig:plan}. Only $\state_A$ and $\state_B$ lie in the possibly unsafe (yellow) region, so if no unmodeled obstacles appear at runtime, we can only be forced to replan in two cases: 1) $\state_A$ is unsafe, 2) $\state_A$ is safe and $\state_B$ is unsafe. In case 1, we update the belief, keeping only constraints marking $\state_A$ as unsafe, and plan a contingency satisfying as many constraints as possible from the new belief. This repeats recursively, up to a finite recursion depth, for any possibly unsafe states on the contingency. In case 2, updating the belief to mark $\state_A$ as safe and $\state_B$ as unsafe renders the belief empty, since this is impossible given the initial belief and box parameterization. We can thus avoid computing further contingencies on this policy tree branch.

\begin{wrapfigure}{l}{0.45\linewidth}
\vspace{-23pt}\begin{table}[H]
\footnotesize	\begin{tabular}{|l||*{2}{c|}}\hline
\diagbox[width=20mm,outerleftsep=17pt,outerrightsep=-10pt]{\hspace{-27pt}Constraints/prior}{Problem\hspace{-2pt}}
&\makebox[4.2em]{\ Prob. MCV }
&\makebox[4.2em]{\ Prob. MEC }\\\hline\hline
\hspace{-3pt}MICP-representable\hspace{-6pt} &\hspace{-3pt}Prob. \ref{prob:cc_riemann}-$\varepsilon_\textrm{min}$\hspace{-3pt} & \hspace{-3pt}Prob. \ref{prob:cc_riemann}-R\hspace{-3pt}\\\hline
\hspace{-3pt}Not MICP-rep. & MCR & BTP\\\hline
\end{tabular}\vspace{-6pt}\caption{Which open-loop planner to use?}\label{table:planners}
\end{table}\vspace{-20pt}
\end{wrapfigure}

\vspace{-3pt}To recap, we would ideally solve Prob. \ref{prob:cc} to get open-loop plans for our policy, but it is intractable. If all constraints (dynamics, uncertain constraints, etc.) and the integrals of $p(\theta)$ are MICP-representable, we can approximate Prob. \ref{prob:cc} with Prob. \ref{prob:cc_riemann}, which enjoys theoretical guarantees by using the \textit{infinite} belief. If not, we use MCR/BTP, which use \textit{finite samples} from the belief. This is summarized in Table \ref{table:planners}.

\section{Experiments}\label{sec:results}

We show our method scales to safely and efficiently plan for high-dimensional (12D) systems with combined state/control constraint uncertainty, constraint sensing uncertainty, and high-dimensional (30D) constraints. See App. \ref{app:results} for more details and experiments (7-DOF arm planning with suboptimal demonstrations, and nonlinear constraint planning using zonotope-based $\feastheta$ extraction), App. \ref{app:computation_time} for computation time discussion, and \url{https://youtu.be/aWZ_U-gWQJI} for visualizations.

\textbf{Mixed quadrotor uncertainty}: We plan for a quadrotor (dynamics in App. \ref{app:results}) carrying a payload of uncertain weight around uncertain obstacles. We are given one demonstration (Fig. \ref{fig:mixed_quad}.B, pink) to learn a 7D constraint $\theta \in \mathbb{R}^7$ (6 for obstacle, 1 for control). After extracting $\feastheta$ with Alg. \ref{alg:extraction}, (Fig. \ref{fig:mixed_quad}.B), we remain uncertain about the obstacle's $y$-extents. We model the uncertain weight as an unknown control constraint $\Vert u \Vert_2^2 \le \bar U(\theta)$; from the demonstration, we learn that $\Vert u \Vert_2^2 \le 97.85$ is guaranteed safe (Fig. \ref{fig:mixed_quad}.A), and KKT also tells us the constraint is inactive, so $\Vert u \Vert_2^2 \in (97.85, 100] = \textrm{proj}_\controlset(\feastheta)$ is possibly unsafe. The quadrotor has bump and torque sensors to directly detect state and control constraint violations, respectively. We now solve Prob. MCV starting from a lower initial state (Fig. \ref{fig:mixed_quad}.C), which we do by solving Prob. \ref{prob:cc_riemann}-$\varepsilon_\textrm{min}$ for an initial plan and contingencies, directly optimizing over the infinite constraint belief. We use a double-integrator approximation of the quadrotor dynamics and restrict each open-loop trajectory to 30 timesteps; thus, it is not possible to satisfy all constraints in $\feastheta$ while reaching the goal in the time limit. Solving Prob. \ref{prob:cc_riemann}-$\varepsilon_\textrm{min}$ returns Plan 1 (Fig. \ref{fig:mixed_quad}.C), which violates some possible control constraints in order to lift the quadrotor over all possible obstacles. Our policy also generates contingencies (Fig. \ref{fig:mixed_quad}.D-F) in case Plan 1 violates the true control constraint, and we must plan to avoid the possible obstacles. To emphasize the benefit of \textit{optimizing} over the infinite set of possible constraints, we compare to two baselines: a variant of the scenario approach \cite{scenario}, where we iteratively \textit{sample} and enforce $\{\theta_i\sim b(\theta)\}_{i=1}^{N_\textrm{sam}}$ until the planning problem becomes infeasible, and an optimistic planner, which only avoids $\guarunsafe$ and replans upon violating a constraint (see App. \ref{app:planning_baselines} for baseline details). We evaluate the number of violations on constraints uniformly drawn from $\feastheta$. Our policy suffers 0.54 $\pm$ 0.94 constraint violations (average $\pm$ standard deviation), the scenario approach 1.30 $\pm$ 1.36 violations, and the optimistic strategy 9.10 $\pm$ 4.65 violations. We outperform the scenario approach, as we \textit{optimize} over the set of constraints to satisfy, and the optimistic strategy, as it ignores constraint uncertainty. Running Alg. \ref{alg:extraction} and Prob. \ref{prob:cc_riemann} takes 3.3 and 1.1 sec., respectively. See App. \ref{app:results_mixed} for a constraint violation histogram empirically validating our probabilistic safety guarantees. 

\begin{figure}[H]
        \centering
        \includegraphics[width=\linewidth]{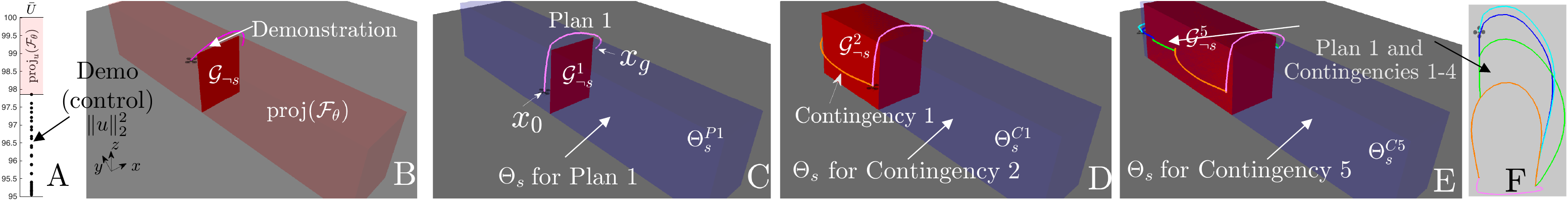}\vspace{-0pt}
        \caption{\small Mixed quadrotor uncertainty example. \textbf{A-B}. Initial control and state constraint uncertainty. \textbf{C}. Initial plan for a new task. \textbf{D-F}. Contingencies are pre-computed, and the system switches if the initial plan is unsafe.\vspace{-0pt}}
        \label{fig:mixed_quad}
\end{figure}

\textbf{7-DOF arm with contact sensing uncertainty}: We plan for a 7-DOF arm (dynamics in App. \ref{app:results}) near a storage rack. We are given two demonstrations (Fig. \ref{fig:arm}.A), and after running Alg. \ref{alg:extraction} to obtain $b_\dem(\theta)$, partly reveal the shelf constraint (Fig. \ref{fig:arm}.B), which has parameters $\theta \in \mathbb{R}^6$. We assume an uncertain contact sensing model \cite{SaundB18}, where contact is assumed to be at any point on the arm downstream on the kinematic chain where a torque limit is exceeded (c.f. App. \ref{app:planning} ). We solve Prob. MEC for $c_\task(\trajxu) = \sum_{t=1}^{T-1} \Vert \state_{t+1} - \state_t \Vert_2$, for the task of moving the arm from below to above the shelf (Fig. \ref{fig:arm}.B), using 100 (uniform) samples from $b_\dem(\theta)$ in BTP. We compare our policy to 1) BTP without demonstrations and the union-of-boxes parameterization, and 2) an optimistic approach that executes the optimal path on the BTP graph after removing unsafe edges, and replanning if a constraint is violated. From the demonstrations, we can determine that a subset of the shelf is guaranteed unsafe (Fig. \ref{fig:arm}, dark red); beyond that, we are uncertain (Fig. \ref{fig:arm}.B). Thus, we plan to avoid that region, swinging the arm around and over the uncertain area. However, in doing so, the arm bumps into an unmodeled obstacle: a box on the lower shelf (Fig. \ref{fig:arm}.C). The contact sensor informs our method that some point on the end effector is in collision; we add 300 sampled points on the end effector to $\collset_{\neg s}$ and samples from the traversed free space to $\collset_{s}$. Our method then automatically determines that it needs to update the constraint parameterization, as geometrically there is no single box that can explain the demonstrations, $\collset_s$, and $\collset_{\neg s}$. After updating $\theta$ to include two boxes (now $\theta \in \mathbb{R}^{12}$), we extract $\feastheta$ for this updated parameterization, and resample 100 (uniform) samples from the updated $b_\textrm{ex}(\theta)$ as input to BTP (see Fig. \ref{fig:arm}.D). As our belief now indicates an uncertain region near the lower shelf obstacle, our policy plans to move further out from the shelf to avoid the uncertain region, and reaches the goal. Running Alg. \ref{alg:extraction} and BTP takes 20 and between 5-20 min. (can be sped up by precomputing arm swept volumes, see \cite{BTP}, App. \ref{app:computation_time}), respectively. Overall, our policy reaches the goal with a cost of 8.19 rad, while the cost without demonstrations/boxes is 18.24 rad, and the optimistic approach is 143.31 rad. Without demonstrations/boxes, we need several more iterations bumping into the shelf before it is sufficiently localized, and the optimistic approach ignores spatial correlation in edge validity, exploring far more edges (c.f. App. \ref{app:results} for details). This example suggests our method scales to high-dimensional systems, can detect when the constraint representation is insufficient, and can use complex constraint measurements.

\begin{figure}[H]
        \centering
        \includegraphics[width=\linewidth]{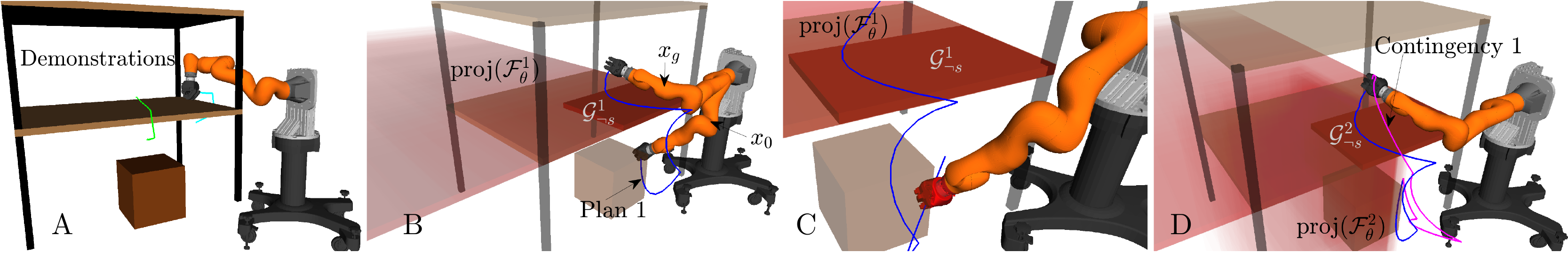}
        \caption{Arm with contact sensing uncertainty. \textbf{A}: Demonstrations. \textbf{B}. Initial constraint uncertainty (red) and plan (blue). \textbf{C}. The initial plan violates an unmodeled constraint, triggering a belief update. \textbf{D}. Replan online.\vspace{-5pt}}
        \label{fig:arm}
\end{figure}

\textbf{Quadrotor maze}: We plan for a quadrotor in clutter with two-meter radius LiDAR sensing (Fig. \ref{fig:quad_maze}). We know the brown obstacles \textit{a priori}, and are given five demonstrations that reveal five obstacles ($\theta \in \mathbb{R}^{30}$) (Fig. \ref{fig:quad_maze}.A), but provide little information about their size. We solve Prob. MEC steering from the bottom to the top of the maze (Fig. \ref{fig:quad_maze}.D) while minimizing $c(\trajxu) = \sum_{t=1}^{T-1} \Vert \control_t \Vert_2^2$ by solving Prob. \ref{prob:cc_riemann}-R, optimizing directly over the continuous $b(\theta)$, and computing contingencies. We also modify Prob. \ref{prob:cc_riemann}-R to never collide in execution by avoiding the set of inevitable collision states \cite{ICS} under the double-integrator model; this is modeled with additional constraints (see App. \ref{app:planning}). We obtain dynamically-feasible quadrotor trajectories by warmstarting the nonlinear optimization with the double-integrator trajectory. We visualize our policy in Fig. \ref{fig:quad_maze}. Running Alg. \ref{alg:extraction} and Prob. \ref{prob:cc_riemann} takes 1 sec. and 1 min., respectively. Plan 1 (pink) intelligently trades off risk and performance. Note there may exist a direct path to the goal between the brown obstacles; however, the orange demonstration induces an obstacle that likely blocks this path. Also, moving left is riskier than moving right, as the uncertain obstacle on the left may create a dead end. Plan 1 avoids both traps, moving to the right and increasing altitude to avoid all possible obstacles induced by the dark blue demonstration. This enables maintenance of higher speed and thus lower cost, instead of cautiously approaching the uncertain region to determine if it is safe to cut through. Finally, Plan 1 cuts through the possibly-unsafe region induced by the green demonstration, as the obstacle is unlikely to extend down to the brown obstacle. We discretize the possible constraint measurements in this region on a grid, planning contingencies if the passage is partially (Contingencies 1-3) or completely blocked (4). We compare to two approaches, \cite{ral}, which plans trajectories which are guaranteed-safe under the constraint parameterization, and \cite{janson}, which plans optimistically over sets of subgoals on the frontier (see App. \ref{app:planning_baselines} for more details). Drawing constraints uniformly from $\feastheta$, our policy solves Prob. MCE with a cost of 1.28 $\pm$ 0.27, while \cite{ral} is conservative, with a cost of 6.29, and \cite{janson} returns a cost of 5.51 $\pm$ 1.65, as it explores the likely dead end between the brown obstacles. This example suggests our method scales to high-dimensional systems and constraint spaces and can compute a policy integrating sensor inputs to adaptively switch between plans and contingencies.

\begin{figure}[!htb]
        \centering
        \includegraphics[width=\linewidth]{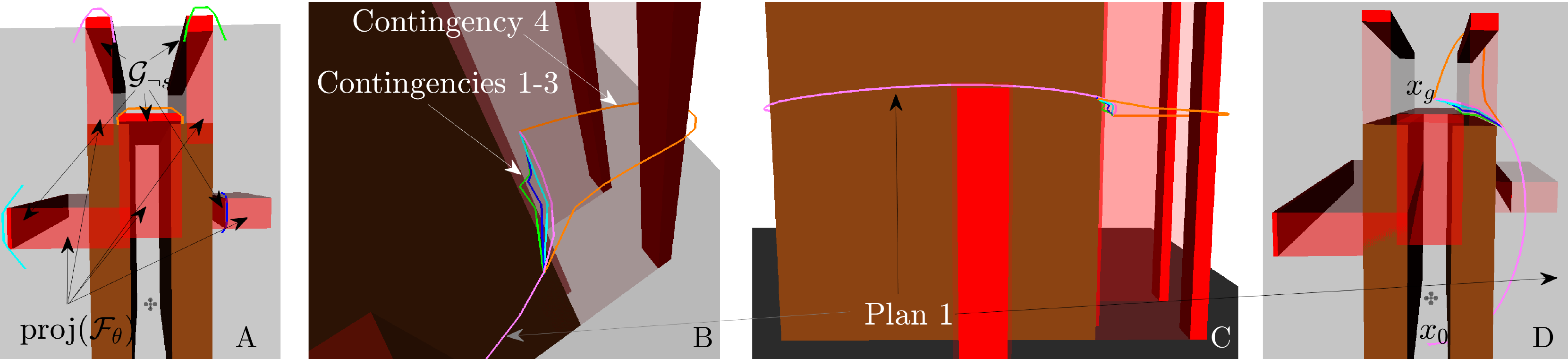}
        \caption{Quadrotor maze. \textbf{A}. Demos., initial constraint uncertainty. \textbf{B-D}. Three views of the initial plan (pink) and contingencies for different sensing possibilities, obtained by gridding the possible sensor measurements.}
        \label{fig:quad_maze}
\end{figure}

\section{Conclusion}

We present a method to address uncertainty in constraints learned from demonstrations. Instead of trying to satisfy all possible constraints, we obtain a belief over constraints, then design open-loop planners which use the belief to ``be as safe as possible while remaining efficient". We use these planners in a closed-loop policy that uses constraint data gathered online to help complete the task. In future work, we aim to speed up our method with parallel extraction and fast integer programming \cite{mip_ms}, and extend our method to adaptively plan with beliefs over temporal logic formulas \cite{rss}.


\acknowledgments{This work was supported in part by a National Defense Science and Engineering Graduate (NDSEG)
Fellowship, Office of Naval Research (ONR) grants N00014-18-1-2501 and N00014-17-1-2050, and
National Science Foundation (NSF) grants ECCS-1553873 and IIS-1750489. We thank Brad Saund for insightful discussions and help on the BTP experiments.}



\bibliography{example}  

\appendix
\newpage

\begin{minipage}{\linewidth}
\centering{\huge{\textbf{Appendix}}}
\end{minipage}

\vspace{10pt}

In these appendices, we will first summarize and provide more details on the optimization problems used in our method (Appendix \ref{app:overview_glossary}, discuss various results on the representability of sets of cost function and constraint parameters which are consistent with demonstrations (Appendix \ref{app:shapes}), provide expanded details on our methods for extracting the set of consistent cost function and constraint parameters (Appendix \ref{app:extraction}), provide expanded details on our methods for planning policies for adaptively satisfying uncertain constraint parameters (Appendix \ref{app:planning}), provide proofs for the theoretical results in the main body (Appendix \ref{app:theory}), and provide additional details on our experimental results (Appendix \ref{app:results}).

\section{Optimization problem glossary}\label{app:overview_glossary}

In this appendix, we provide a detailed summary of the optimization problems utilized in our approach.

\textbf{Problem \ref{prob:fwd_prob}}: This is the optimization problem that we assume the demonstrator is solving to local-optimality. This problem involves a potentially task-dependent cost function $c_\task(\trajxu)$, where a task in this paper is simply steering the system state from a start state $\state_0$ to a goal state $\state_g$ while satisfying a set of constraints. This problem also involves a known shared constraint $\bar\phi(\trajxu) \in \bar\safeset$ (which embeds known constraints which that shared across all tasks, such as the system dynamics) as well as a known task-dependent constraint $\phi_\task(\trajxu) \in \safeset_\task$ (which embeds known constraints that are task-dependent, such as the start and goal state constraints). Finally, there is the unknown shared constraint $\phi(\trajxu) \in \safeset(\theta)$ (unknown to the learner, but known to the demonstrator) which is parameterized by unknown parameters $\theta$.

\begin{equation*}\label{eq:fwdprob}
	\begin{array}{>{\displaystyle}c >{\displaystyle}l >{\displaystyle}l}
				&\\[-15pt]
		\underset{\trajxu}{\text{minimize}} & \quad c_\task(\trajxu) &\\
		\text{subject to} & \quad \phi(\trajxu) \in \safeset(\theta) \subseteq \constraintspace & \Leftrightarrow\quad \mathbf{g}_{\neg k}(\trajxu, \theta) \le \mathbf{0}\\
		& \quad \bar\phi(\trajxu) \in \bar\safeset \subseteq \bar\constraintspace, \quad \phi_\task(\trajxu) \in \safeset_\task \subseteq \constraintspace_\task & \Leftrightarrow \quad \mathbf{h}_k(\trajxu) = \mathbf{0},\quad \mathbf{g}_{k}(\trajxu) \le \mathbf{0}
	\end{array}\hspace{-15pt}
\end{equation*}

\textbf{Problem \ref{prob:kkt_opt}}: This is the inverse optimization problem that the learner solves to learn \textit{one possible assignment} of the unknown constraints that are satisfied by the demonstrations. Specifically, the problem searches over the unknown constraint parameters and the Lagrange multipliers which together make the KKT conditions of each demonstration satisfied.

\begin{equation*}\vspace{2pt}\label{eq:fwdprob}
	\hspace{-27pt}\begin{array}{>{\displaystyle}c >{\displaystyle}l >{\displaystyle}l}
				&\\[-18pt]
		\text{find} & \theta, \lag \doteq \{\boldsymbol{\lambda}_{k}^j, \boldsymbol{\lambda}_{\neg k}^j,\boldsymbol{\nu}_{k}^j\}_{j=1}^{\numsafe}\\
		\text{subject to} & \{\textrm{KKT}(\traj_{j}^\textrm{loc})\}_{j=1}^{\numsafe}\\[-2pt]
	\end{array}\hspace{-20pt}
\end{equation*}

\textbf{Problem \ref{prob:kkt_opt_robust}}: This problem returns the set of all consistent constraint parameters $\feastheta$. Intuitively, this problem finds the largest set $\feasrobust$, in the sense of set containment, of constraint parameters $\theta$, such that the KKT conditions can be made to hold for those parameters. The problem is intractable in its most general form, motivating simpler variants Problem \ref{prob:kkt_opt_robust_boxes} and Problem \ref{prob:kkt_opt_robust_zonotopes}.
\begin{equation*}\vspace{2pt}\label{eq:fwdprob}
	\hspace{-27pt}\begin{array}{>{\displaystyle}c >{\displaystyle}l >{\displaystyle}l}
				&\\[-17pt]
		\underset{\feasrobust}{\text{sup}} & \textrm{Vol}(\feasrobust) \\[-4pt]
		\text{s.t.} & \forall \theta \in \hat\feas_\theta,\ \ \exists \{\boldsymbol{\lambda}_{k}^j, \boldsymbol{\lambda}_{\neg k}^j,\boldsymbol{\nu}_{k}^j \mid \textrm{KKT}(\traj_{j}^\textrm{loc})\}_{j=1}^{\numsafe}\\[-4pt]
	\end{array}\hspace{-20pt}
\end{equation*}

\textbf{Problem \ref{prob:kkt_opt_robust_boxes}}: This problem returns the largest axis-aligned hyper-rectangle contained within $\feastheta$; this problem is a restricted, tractable version of Problem \ref{prob:kkt_opt_robust}.

\begin{equation*}\vspace{2pt}\label{eq:fwdprob}
	\hspace{-33pt}\begin{array}{>{\displaystyle}c >{\displaystyle}l >{\displaystyle}l}
				&\\[-13pt]
		\underset{s, \theta, \lag }{\text{maximize}} & \big(\textstyle\prod_i s_i\big)^{1/d} \\
		\text{subject to} & \{\textrm{KKT}_\textrm{rob}^\textrm{box}(\traj_{j}^\textrm{loc})\}_{j=1}^{\numsafe}\\[1pt]
	\end{array}\hspace{-20pt}
\end{equation*}

\textbf{Problem \ref{prob:cc}}: This problem solves a chance-constrained trajectory optimization problem, minimizing a possibly task-dependent objective $c_\task(\trajxu)$ while ensuring that the resulting trajectory satisfies the uncertain constraint with prescribed probability $1-\varepsilon$. We further consider two specific variants, \textbf{Problem \ref{prob:cc}-$\varepsilon_\textrm{min}$}, which seeks to solve Problem \ref{prob:cc} to be as safe as possible, i.e with the smallest $\varepsilon$ for which there exists a feasible solution, and \textbf{Problem \ref{prob:cc}-R}, which directly trades off performance and safety with a modified objective $c_\task(\trajxu)/\pr(\trajxu\safe)$. Problem \ref{prob:cc} and its variants are in their most general form intractable, motivating the simpler variant Problem \ref{prob:cc_riemann}.

\textbf{Problem \ref{prob:cc}:}
\begin{subequations}
	\begin{align}
    \min_{\trajxu }\quad & c_{\task}(\trajxu)\label{eq:cc1} \\
    \textrm{s.t.}\quad  & \bar\phi(\trajxu) \in \bar\safeset \subseteq \bar\constraintspace\label{eq:cc2} \\
	& \phi_{\task}(\trajxu) \in \safeset_{\task} \subseteq \constraintspace_{\task}\label{eq:cc3} \\
    & \textrm{Pr}(\trajxu \textrm{ safe}) \ge 1-\varepsilon\label{eq:cc4}
\end{align}
\end{subequations}

\textbf{Problem \ref{prob:cc}-$\varepsilon_\textrm{min}$:}
\begin{subequations}
	\begin{align*}
    \min_{\trajxu }\min_\varepsilon\quad & c_{\task}(\trajxu) \\
    \textrm{s.t.}\quad  & \bar\phi(\trajxu) \in \bar\safeset \subseteq \bar\constraintspace\\
	& \phi_{\task}(\trajxu) \in \safeset_{\task} \subseteq \constraintspace_{\task} \\
    & \textrm{Pr}(\trajxu \textrm{ safe}) \ge 1-\varepsilon
\end{align*}
\end{subequations}

\textbf{Problem \ref{prob:cc}-R:}
\begin{subequations}
	\begin{align*}
    \min_{\trajxu }\quad & c_{\task}(\trajxu)/\textrm{Pr}(\trajxu \textrm{ safe}) \\
    \textrm{s.t.}\quad  & \bar\phi(\trajxu) \in \bar\safeset \subseteq \bar\constraintspace\\
	& \phi_{\task}(\trajxu) \in \safeset_{\task} \subseteq \constraintspace_{\task}
\end{align*}
\end{subequations}

\textbf{Problem \ref{prob:cc_riemann}}: This is a simplified variant of Problem \ref{prob:cc}, which makes the probability constraint tractable by restricting the integration of probability mass over a fixed number of axis-aligned boxes, where the location and extents of the boxes are also optimized over. We consider two specific variants, \textbf{Problem \ref{prob:cc_riemann}-$\varepsilon_\textrm{min}$} and \textbf{Problem \ref{prob:cc_riemann}-R}, which are as described for Problem \ref{prob:cc}.

\begin{subequations}
	\begin{align}
    \hspace{-5pt}\min_{\trajxu, \mathcal{B}_i, t_i }\ & c_\task(\trajxu)\label{eq:riemann1}\\[-4pt]
    \textrm{s.t.}\quad\  & \bar\phi(\trajxu) \in \bar\safeset \subseteq \bar\constraintspace,\ \phi_\task(\trajxu) \in \safeset_\task \subseteq \constraintspace_\task\label{eq:riemann2} \\
	& \trajxu \in \safeset(\theta),\ \forall \theta \in \mathcal{B}_1, \ldots, \mathcal{B}_{N_\textrm{box}}\label{eq:riemann3} \\
    & \mathcal{B}_i \cap \mathcal{B}_j = \emptyset, \ i \ne j,\quad \mathcal{B}_i \subseteq \feastheta, \forall i\label{eq:riemann4} \\
    & 0 \le t_i \le (\textstyle\prod_i b_i^\textrm{scale})^{1/d}, \ i = 1,..., N_\textrm{box}\label{eq:riemann5}\\
    & \textstyle\sum_i t_i^{d} \ge (1-\varepsilon)\textrm{Vol}(\feastheta)\label{eq:riemann6}
\end{align}
\end{subequations}

We provide an overview of the constraints of Problem \ref{prob:cc_riemann}. First, note that \eqref{eq:riemann1}-\eqref{eq:riemann2} exactly correspond to \eqref{eq:cc1}-\eqref{eq:cc3}. The remaining constraints in Problem \ref{prob:cc_riemann} implement the box-limited integration. Specifically, \eqref{eq:riemann3} enforces that the planned trajectory $\trajxu$ is safe with respect to all $\theta$ belonging to $\mathcal{B}_1,\ldots,\mathcal{B}_{N_\textrm{box}}$. Recall that each $\mathcal{B}_i$ is meant to represent a box contained in $\feastheta$, and that $\trajxu$ is to be safe with respect to all $\theta$ belonging to this $\mathcal{B}_i$. Thus, \eqref{eq:riemann4} further enforces that each $\mathcal{B}_i$ is contained in $\feastheta$, and furthermore, that the $\mathcal{B}_i$ are disjoint; this is to avoid any double-counting of box volumes (and thus probability mass). Next, \eqref{eq:riemann5} introduces the variables $t_i$, from which the volume of the corresponding box can be recovered: note that the volume of $\mathcal{B}_i$ is $\prod_i b_i^\textrm{scale}$. The last constraint, \eqref{eq:riemann6}, does exactly this: $t_i^d$ is exactly the volume of $\mathcal{B}_i$, so $\sum_i t_i^d = \sum_i \textrm{Vol}(\mathcal{B}_i)$, which we ensure is at least $(1-\varepsilon)\textrm{Vol}(\feastheta)$ to satisfy the probability constraint.

\textbf{Problem \ref{prob:kkt_opt_robust_zonotopes}}: In Appendix \ref{app:extraction_zon}, we will discuss a modification of Algorithm \ref{alg:extraction} which makes it more efficient for constraint parameterizations other than the union-of-boxes parameterization assumed in Sec. \ref{sec:extraction}. This modification hinges upon Problem \ref{prob:kkt_opt_robust_zonotopes}, which returns a zonotope contained within $\feastheta$ of approximately maximum volume; this problem is a restricted, tractable version of Problem \ref{prob:kkt_opt_robust} which is more general than Problem \ref{prob:kkt_opt_robust_boxes}.

\begin{equation*}\vspace{2pt}
	\hspace{-27pt}\begin{array}{>{\displaystyle}c >{\displaystyle}l >{\displaystyle}l}
				&\\[-19pt]
		\underset{s, \theta, \boldsymbol{\lambda}_{k}^j, \boldsymbol{\lambda}_{\neg k}^j,\boldsymbol{\nu}_{k}^j, Q_i}{\text{maximize}} & \textstyle\sum_{i=1}^{N_\textrm{gen}} \Vert \ell_i \Vert_1 \\[-2pt]
		\text{subject to} & \{\textrm{KKT}_\textrm{rob}^\textrm{zon}(\traj_{j}^\textrm{loc})\}_{j=1}^{\numsafe}, \quad |\ell_m^\top \ell_n| \le \delta,\ \forall m \ne n
	\end{array}\hspace{-20pt}
\end{equation*}

\section{A geometric analysis of constrained inverse optimal control}\label{app:shapes}

We describe how the constraint learning problem can be extended to also learn unknown cost function parameters $\thetac$ (Appendix \ref{app:costfn}), the shape of the resulting feasible sets for unknown cost function parameters for various parameterizations (Appendix \ref{app:costfnshape}), and the shape of the resulting feasible sets for consistent constraint parameters (Appendix \ref{app:constraintsshape}), for various parameterizations.

\subsection{Modifying Problem \ref{prob:kkt_opt} to handle unknown cost function parameters}\label{app:costfn}

As in \cite{ral}, we note that the KKT conditions in Problem \ref{prob:kkt_opt} can be modified to handle unknown cost function parameters, where the cost function can be written as $c_\task(\trajxu, \thetac)$ for unknown cost function parameters $\thetac\in\Thetac$, with few changes: the only KKT condition that changes is stationarity \eqref{eq:kkt_stat}, where the term involving the gradient of the cost now also involves the unknown cost function parameters $\thetac$:
\begin{equation}
	\nabla_{\trajxu} c_\task(\traj_{j}^\textrm{loc}, \textcolor{red}{\thetac}) + \textcolor{blue}{\boldsymbol{\lambda}_{k}^{j}}^\top \nabla_{\trajxu} \mathbf{g}_{k}(\traj_{j}^\textrm{loc})+\textcolor{blue}{\boldsymbol{\lambda}_{\neg k}^j}^{\hspace{-4pt}\top} \nabla_{\trajxu} \mathbf{g}_{\neg k}(\traj_{j}^\textrm{loc}, \textcolor{red}{\theta})+\textcolor{blue}{\boldsymbol{\nu}_{k}^j}^\top \nabla_{\trajxu} \mathbf{h}_{k}(\traj_{j}^\textrm{loc}) = \mathbf{0}
	\label{eq:kkt_stat_cost}
\end{equation}

\subsection{Unknown cost function, known constraint}\label{app:costfnshape}

Let us denote the feasible set of Problem \ref{prob:kkt_opt}, modified to handle unknown cost function parameters, again be $\feas$, and let us denote the projection of $\feas$ onto $\Thetac$ as $\feasgamma$:
\begin{equation}
	\feas_\thetac \doteq \{ \thetac \mid \exists (\theta,\lag): (\theta,\thetac,\lag) \in \feas \}
\end{equation}

In the following, we will analyze the shape of $\feasgamma$ for various parameterizations the unknown cost function. In this subsection, we will assume that the constraint parameters $\theta$ are known, and focus only on analyzing the case of unknown $\thetac$.

\subsubsection{Linear cost function parameterization}
Consider the case where the cost function $c(\trajxu, \thetac)$ is linear in the unknown parameters $\thetac$, i.e. $c(\trajxu, \thetac) = \sum_{i=1}^{|\thetac|} \thetac_i c_i(\trajxu) \doteq \thetac^\top c(\trajxu)$, and the constraints are fully known. The following result shows that in this setting, the set of cost function parameters consistent with the KKT conditions \eqref{eq:kkt} is convex and closed under nonnegative scaling:
\begin{theorem}[Geometry of $\feas_\thetac$ (linear in $\thetac$, known $\theta$)]\label{thm:cone}
	If the cost function takes the form $c(\trajxu, \thetac) = \sum_{i=1}^{|\thetac|} \thetac_i c_i(\trajxu)$, then $\feas_\thetac = \Thetac \cap \mathcal{C}$, where $\mathcal{C} \in \mathbb{R}^{|\thetac|}$ is a convex cone in $\thetac$-space.
\end{theorem}
\begin{proof}
	From \cite{cvxbook}, a set $\mathcal{C}$ is a convex cone if for all $\thetac_1$, $\thetac_2$ in $\mathcal{C}$, $\alpha_1\thetac_1 + \alpha_2\thetac_2 \in \mathcal{C}$, for all nonnegative scalars $\alpha_1$, $\alpha_2$ $\ge 0$. For now, assume that other than the KKT constraints, $\thetac$ is unconstrained, i.e. $\Thetac = \mathbb{R}^{|\thetac|}$. Suppose we have $\thetac_1 \in \feas_\thetac$ and $\thetac_2 \in \feas_\thetac$. In \eqref{eq:kkt}, as the constraint parameters $\theta$ are known, we drop \eqref{eq:kkt_primal3}, merge \eqref{eq:kkt_comp1}-\eqref{eq:kkt_comp2}, and absorb $\boldsymbol{\lambda}_{\neg k}^{j\top} \nabla_{\trajxu} \mathbf{g}_{\neg k}(\traj_{j}^\textrm{loc}, \theta)$ into the previous term. Then, if $\thetac_i \in \feas_\thetac$ for $i\in\{1,2\}$, we know that $\boldsymbol{\lambda}_{k,i}^j \ge \mathbf{0}$, $\boldsymbol{\lambda}_{k,i}^j \odot \mathbf{g}_{k}(\demj) = \mathbf{0}$, and $\gamma_i^\top \nabla_{\trajxu} c(\traj_{j}^\textrm{loc}) + \boldsymbol{\lambda}_{k,i}^{j\top} \nabla_{\trajxu} \mathbf{g}_{k}(\traj_{j}^\textrm{loc})+\boldsymbol{\nu}_{k,i}^{j\top} \nabla_{\trajxu} \mathbf{h}_{k}(\traj_{j}^\textrm{loc}) = \mathbf{0}$, for $i\in\{1,2\}$. Then if $\thetac = \alpha_1\thetac_1 + \alpha_2\thetac_2$, for nonnegative scalars $\alpha_1$ and $\alpha_2$, we can select $\boldsymbol{\lambda}_k^j = \alpha_1 \boldsymbol{\lambda}_{k,1}^j + \alpha_2 \boldsymbol{\lambda}_{k,2}^j$ and $\boldsymbol{\nu}_k^j = \alpha_1 \boldsymbol{\nu}_{k,1}^j + \alpha_2 \boldsymbol{\nu}_{k,2}^j$ to satisfy \eqref{eq:kkt_stat}; it can also be verified algebraically that this choice of $\boldsymbol{\lambda}_k^j$ satisfies the nonnegativity and complementary slackness constraints. This implies the conic hull $\mathcal{C}$ of any feasible $\thetac$ is feasible for Problem \ref{prob:kkt_opt}. Finally, if $\Thetac \subset \mathbb{R}^{|\gamma|}$, we can write $\feas_\thetac$ as the intersection of $\Thetac$ and the previously constructed cone $\mathcal{C}$.
\end{proof}
Furthermore, as Problem \ref{prob:kkt_opt} simplifies to a linear program when $\theta$ is known, $\feas$ is a polytope, and thus $\feas_\thetac$ can be directly computed via a polytopic projection of $\feas$ onto its $\thetac$ coordinates \cite{mpt3}.

\subsubsection{Nonlinear cost function parameterization}
For general nonlinear parameterizations of $\thetac$, the set of $\thetac$ which satisfy \eqref{eq:kkt} is much more challenging to represent explicitly, as checking if a $\thetac$ satisfies \eqref{eq:kkt} will involve satisfying a set of nonlinear, non-convex equality constraints \eqref{eq:kkt_stat}. However, if the parameterization is a polynomial function of $\thetac$, i.e. $c(\trajxu, \gamma)$ is a polynomial in $\gamma$ for fixed $\trajxu$, $\feas_\thetac$ can be represented as a semi-algebraic set; that is, a set described by a finite union of intersections of polynomial inequalities:
\begin{theorem}
	For cost functions which are polynomial in $\thetac$ for fixed $\trajxu$, $\feas_\thetac = \Thetac \cap \mathcal{P}$, where $\mathcal{P} \in \mathbb{R}^{|\thetac|}$ is a semi-algebraic set in $\thetac$-space. 
\end{theorem}
\begin{proof}
	In this setting, $\boldsymbol{\lambda}_{k}^j \ge \mathbf{0}$, $\boldsymbol{\lambda}_{k}^j \odot \mathbf{g}_{k}(\demj) = \mathbf{0}$, and $\nabla_{\trajxu} c(\traj_{j}^\textrm{loc}, \thetac) + \boldsymbol{\lambda}_{k}^{j\top} \nabla_{\trajxu} \mathbf{g}_{k}(\traj_{j}^\textrm{loc})+\boldsymbol{\nu}_{k}^{j\top} \nabla_{\trajxu} \mathbf{h}_{k}(\traj_{j}^\textrm{loc}) = \mathbf{0}$ define an intersection of polynomial inequalities in $\thetac$, that is, a basic semi-algebraic set. $\feas_\thetac$ is then the projection of this set onto the $\thetac$ coordinates; however, the projection of a basic semi-algebraic set may not be basic semi-algebraic in general; they are guaranteed to be semi-algebraic via the Tarski-Seidenberg theorem \cite{Bochnak}.
\end{proof}

While explicitly representing a semi-algebraic set can be expensive, there exist well-established methods for doing so, including exact methods like cylindrical algebraic decomposition \cite{Collins75} and approximate methods involving semidefinite relaxations \cite{MagronHL15}.

\subsection{Unknown constraints}\label{app:constraintsshape}

In this section, we discuss details on the shape of $\feastheta$ for unions of offset-parameterized constraints (Sec. \ref{sec:unionsofoffset}) and for unions of affine and higher-order parameterized constraints (Sec. \ref{sec:unionsofaffine}).

\subsubsection{Unions of offset-parameterized constraints}\label{sec:unionsofoffset}

\textbf{Coordinate-independent parameterization}: 

We first analyze the case where the unknown constraint can be described as a union of offset-parameterized constraints:

\begin{theorem}\label{thm:unions_of_boxes}
	Consider the case when the unknown constraint can be described as a union of intersection of inequalities which are offset-parameterized, i.e.:
\begin{equation}\label{eq:unionsofintersections}
	\safeset(\theta) = \Bigg\{\cstate \in \constraintspace \mid \bigvee_{m=1}^{N_\ineq} \bigwedge_{n=1}^{N_\ineq^i} \Big( g_{mn}(\cstate) \le \theta_{mn} \Big) \Bigg\}
\end{equation}

Then, the corresponding $\feastheta$ can be described as a union of boxes as well:
\begin{equation}\label{eq:F_unionsofboxes}
	\feastheta = \Bigg\{\theta \in \Theta \mid \bigcup_{m=1}^{N_\tbox} [I_{d\times d}, -I_{d\times d}]^\top \theta \le s_m \Bigg\}
\end{equation}
\end{theorem}
\begin{proof}
	In this setting, note that Problem \ref{prob:kkt_opt} can be represented as a MILP, implying that $\feas$ can be described as a finite union of polyhedra in the space of all decision variables. As polyhedra are closed under projection, $\feas_\theta$ can also be represented as a finite union of polyhedra. Note that in the KKT conditions for parameterization \eqref{eq:unionsofintersections}, $\theta$ only appears in \eqref{eq:kkt_primal3} and \eqref{eq:kkt_comp2} (as the gradient term in \eqref{eq:kkt_stat} drops out). Now, suppose we fix all of the boolean variables in Problem \ref{prob:kkt_opt} (needed to implement \eqref{eq:kkt_primal3} and \eqref{eq:kkt_comp2}). Then, depending on those boolean variables for some $t$, $j$, \eqref{eq:kkt_comp2} either enforces $\theta_{mn} = g_{mn}(\cstate_t^j)$ or leaves $\theta_{mn}$ unconstrained, and \eqref{eq:kkt_primal3} imposes $\theta_{mn} \ge g_{mn}(\cstate)$ or leaves $\theta_{mn}$ unconstrained. Then, for fixed boolean variables, for any $m$ and $n$, we obtain linear constraints on $\theta_{mn}$, independent of other $\theta_{m'n'}$ for $m' \ne m$, $n' \ne n$; that is, $\theta_{mn}$ is constrained to lie in an interval. Then, unioning over all feasible boolean assignments, we obtain that $\theta_{m'n'}$ can be described as a finite union of intervals. As a result, $\feas_\theta$ can be described as a union of boxes in $\theta$ space, as in \eqref{eq:F_unionsofboxes}.
\end{proof}

\textbf{Coordinate-dependent parameterization}: If instead the constraint is parameterized such that any single constraint can depend on multiple parameters:
\begin{equation}\label{eq:unionsofintersections_dep}
	\safeset(\theta) = \Bigg\{\cstate \in \constraintspace \mid \bigvee_{m=1}^{N_\ineq} \bigwedge_{n=1}^{N_\ineq^i} \Big( g_{mn}(\cstate) \le \omega_{mn}^\top \theta \Big) \Bigg\}
\end{equation}
for some fixed mixing coefficients $\omega_{mn}$, $\feas_\theta$ can only be represented as the more general union of polytopes, as for some $m$, $n$, the constraints on $\theta_{mn}$ will in general depend on $\theta_{m'n'}$, for $m'\ne m$, $n'\ne n$.

\subsubsection{Unions of affine and higher-degree parameterized constraints}\label{sec:unionsofaffine}

Consider the case where the constraint can be represented as:
\begin{equation}\label{eq:unionsofaffine}
	\safeset(\theta) = \Bigg\{\cstate \in \constraintspace \mid \bigvee_{m=1}^{N_\ineq} \bigwedge_{n=1}^{N_\ineq^i} \Big( g_{mn}(\cstate, \theta_{mn}) \le 0 \Big) \Bigg\}
\end{equation}
where $g_{mn}(\cstate, \theta_{mn})$ is affine or higher-degree in $\theta_{mn}$. In this case, the gradient term does not drop out in \eqref{eq:kkt_stat}, meaning the set of consistent $\theta$ will depend on the projection of a set defined by polynomial equality constraints, which in general is a semialgebraic set described by polynomials of degree $2^{O(|\textrm{decision variables}|)}$, where ``decision variables" denotes the Lagrange multipliers $\lag$ and unknown constraint parameters $\theta$ in Problem \ref{prob:kkt_opt}.

\section{Obtaining a belief over constraints (expanded)}\label{app:extraction}

In this section, we first discuss details on zonotope extraction (Appendix \ref{app:extraction_zon}), how extraction can be done for the case of jointly unknown cost function and constraints (Appendix \ref{app:mixed_costconstraint}), how extraction can be sped up with parallelization (Appendix \ref{app:parallel}), and conclude with a summary of complexity and representability for the extraction problems induced by various cost and constraint parameterizations (Appendix \ref{app:extraction_summary}).

\vspace{-11pt}
\subsection{Other constraint parameterizations: extracting with zonotopes}\label{app:extraction_zon}
\vspace{-3pt}

While filling $\feastheta$ with boxes may be efficient for a union-of-boxes constraint parameterization, infinitely many boxes may be needed to cover $\feastheta$ in more general cases where $\feastheta$ may only be representable as a union of polytopes or semialgebraic sets (see Appendix \ref{app:extraction} for more details). To address this, we describe how to extract $\feastheta$ with shapes more general than boxes while retaining efficiency. Covering $\feastheta$ with polytopes instead is expensive, as polytope volume computation in high dimensions is hard. Instead, we cover $\feastheta$ with zonotopes \cite{Beck2015}, which are between boxes and polytopes in representational power. A zonotope $\mathcal{Z}$ is a Minkowski sum of $N_\textrm{gen}$ line segments: $\mathcal{Z} = \{\sum_{i=1}^{N_\textrm{gen}} \ell_i u_i \mid u_i \in [-1,1] \}$, where $\ell_i \in \mathbb{R}^{d}$ is the $i$th segment.

\begin{wrapfigure}{r}{0.4\linewidth}\vspace{-17pt}
\begin{problem}[Zonotope robustification]\label{prob:kkt_opt_robust_zonotopes}
\normalfont\begin{equation*}\vspace{2pt}
	\hspace{-27pt}\begin{array}{>{\displaystyle}c >{\displaystyle}l >{\displaystyle}l}
				&\\[-19pt]
		\underset{s, \theta, \boldsymbol{\lambda}_{k}^j, \boldsymbol{\lambda}_{\neg k}^j,\boldsymbol{\nu}_{k}^j, Q_i}{\text{maximize}} & \textstyle\sum_{i=1}^{N_\textrm{gen}} \Vert \ell_i \Vert_1 \\[-2pt]
		\text{subject to} & \{\textrm{KKT}_\textrm{rob}^\textrm{zon}(\traj_{j}^\textrm{loc})\}_{j=1}^{\numsafe}\\
		& |\ell_m^\top \ell_n| \le \delta,\ \forall m \ne n
	\end{array}\hspace{-20pt}
\end{equation*}
\end{problem}\vspace{-30pt}
\end{wrapfigure}
Robustifying KKT to a zonotope uncertainty is done similarly to boxes: the robust constraint can be simplified with these equivalences: $a^\top (x + \sum_{i=1}^{N_\textrm{gen}} \ell_i u_i) \le b, \forall u_i \in [-1, 1] \Leftrightarrow \max_{u_1\in[-1,1],\ldots,u_{N_\textrm{gen}}\in[-1,1]} a^\top (x + \sum_i \ell_i u_i ) \le b \Leftrightarrow a^\top x + \sum_i | a^\top \ell_i | \le b$. Denote \eqref{eq:kkt_comp1} and the robustified \eqref{eq:kkt_primal3}, \eqref{eq:kkt_comp2}, \eqref{eq:kkt_stat} as $\textrm{KKT}_\textrm{rob}^\textrm{zon}(\demj)$. Optimizing zonotope volume is challenging, as it requires determinant computations \cite{Beck2015} that render the overall problem a mixed integer semidefinite program, which lack reliable solvers. Instead, we optimize a surrogate, $\sum_i \Vert \ell_i \Vert_1$, and add bilinear optimization constraints to make the lines approximately orthogonal: $|\ell_m^\top \ell_n| \le \delta$ for some small predetermined $\delta$; these constraints are compatible with off-the-shelf solvers \cite{gurobi}. Finally, we cannot ignore the existential quantifiers for this parameterization, so we introduce ``feedback" Lagrange multipliers. Inspired from adjustable robust optimization \cite{Ben-TalGGN04}, we modify each Lagrange multiplier to take the form $\lambda_i + Q_i u$, where $Q_i u$ is a feedback term adjusting the value of the Lagrange multiplier as a linear function of the uncertainty $u$. The $Q_i$ are jointly optimized to maximize the volume. The overall problem (Prob. \ref{prob:kkt_opt_robust_zonotopes}) is a mixed integer bilinear program (MIBLP).

\textbf{Discussion on volume maximization}: Recall from the statement of Problem \ref{prob:kkt_opt_robust_zonotopes} that we are enforcing the approximate orthogonality $|\ell_m^\top \ell_n| \le \delta$ of the line segment generators together with maximizing the line segment norms as a surrogate for volume maximization; this is to avoid the degenerate case where any two generators are parallel; this leads to the extracted volume being zero. Furthermore, we elect to use a prespecified $\delta$ instead of enforcing mutual orthogonality $\ell_m^\top \ell_n = 0$ to avoid restricting the search to rotated boxes (which is what occurs if all the generators are perfectly orthogonal).

\subsection{Discussion on extracting with mixed cost function and constraint uncertainty}\label{app:mixed_costconstraint}

Extraction with mixed cost function and constraint uncertainty can be done in a similar way to Alg. \ref{alg:extraction}. Specifically, we can robustify the stationarity condition to uncertainties in $\thetac$ and $\theta$ jointly:

\begin{equation}\footnotesize
	\nabla_{\trajxu} c_\task(\traj_{j}^\textrm{loc}, \textcolor{red}{\thetac + s_c \odot u_c} ) + \textcolor{blue}{\boldsymbol{\lambda}_{k}^{j}}^\top \nabla_{\trajxu} \mathbf{g}_{k}(\traj_{j}^\textrm{loc})+\textcolor{blue}{\boldsymbol{\lambda}_{\neg k}^j}^{\hspace{-4pt}\top} \nabla_{\trajxu} \mathbf{g}_{\neg k}(\traj_{j}^\textrm{loc}, \textcolor{red}{\theta + s \odot u})+\textcolor{blue}{\boldsymbol{\nu}_{k}^j}^\top \nabla_{\trajxu} \mathbf{h}_{k}(\traj_{j}^\textrm{loc}) = \mathbf{0}
\end{equation}

The remaining KKT conditions are unchanged, as $\thetac$ does not factor into the other constraints. The modified stationarity condition can be robustified in a similar way. We denote \eqref{eq:kkt_comp1}, the robustified \eqref{eq:kkt_primal3} and \eqref{eq:kkt_comp2}, and the joint cost/constraint-robustified \eqref{eq:kkt_stat} together as $\textrm{KKT}_\textrm{rob,cost}^\textrm{box}(\demj)$. We can then modify Problem \ref{prob:kkt_opt_robust_boxes} to account for the additional scaling variables as such:

\begin{equation*}\vspace{2pt}
	\hspace{-0pt}\begin{array}{>{\displaystyle}c >{\displaystyle}l >{\displaystyle}l}
				&\\[-20pt]
		\underset{s, s_c, \theta, \thetac, \lag }{\text{maximize}} & \big(\textstyle\prod_i s_i \textstyle\prod_i s_{c_i} \big)^{1/d} \\
		\text{subject to} & \{\textrm{KKT}_\textrm{rob,cost}^\textrm{box}(\traj_{j}^\textrm{loc})\}_{j=1}^{\numsafe}\\[1pt]
	\end{array}\hspace{-20pt}
\end{equation*}

This new optimization problem can be integrated into Algorithm \ref{alg:extraction}, repeatedly carving out subsets of $\feastheta \times \feasgamma$.

\subsection{Speeding up extraction with parallelization}\label{app:parallel}

We sketch one possible way that Alg. \ref{alg:extraction} can be sped up with parallelization on $M$ cores:
\begin{itemize}
	\item Partition the parameter space $\Theta$ into $M$ disjoint boxes.
	\item Run Alg. \ref{alg:extraction} on each partition separately.
	\item Reconstruct $\feastheta$ by unioning the extracted parameters from each partition.
\end{itemize}

\subsection{Summary on problem complexity}\label{app:extraction_summary}
In the following table, we organize the representability of particular constraint learning and feasible set extraction problems, for various constraint parameterizations. To summarize, the case of unknown constraints induces a set of integer variables due to the complementary slackness condition. The remaining complexity depends on the complexity of the constraint and cost function parameterizations. Furthermore, the extraction problems are only conic-representable due to our formulation of volume maximization.

	\begin{centering}\vspace{4pt}
\begin{tabular}{ c |c|c|c}
 Constraint parameterization & Cost parameterization & Problem \ref{prob:kkt_opt}, class & Class (extraction) \\\hline\hline
 Known & Linear & LP & Direct projection \\\hline
 Union of offsets & Known & MILP & MISOCP\\\hline
 Union of offsets & Linear & MILP & MISOCP\\\hline
 Union of affine & Known & MIBLP & MIBLP\\\hline
 Union of affine & Linear & MIBLP & MIBLP\\\hline
 Nonlinear & Nonlinear & MINLP & MINLP\\\hline
\end{tabular}\vspace{4pt}
\end{centering}

\section{Policies for adaptive constraint satisfaction (expanded)}\label{app:planning}

In this appendix, we first discuss fast reformulations for the chance-constrained planning problem (Appendix \ref{app:fastprob7}), expanded details on the sampling-based planners that we use when the optimization constraints are not MICP-representable (Appendix \ref{app:sampling_based_planners}), discussion on extending Problem \ref{prob:cc_riemann} to handle priors other than the uniform distribution (Appendix \ref{app:planning_priors}), and discussion on how to perform belief updates for various constraint sensing modalities (Appendix \ref{app:planning_update}).

\vspace{-7pt}
\subsection{Fast reformulations of Problem \ref{prob:cc_riemann}}\label{app:fastprob7}\vspace{-3pt}

As written, Problem \ref{prob:cc_riemann} can be expensive to solve due to the many possible assignments of the box decision variables $b_i^\textrm{cen}$ and $b_i^\textrm{scale}$, for each $i$, and the combinatorial coupling between boxes $i \ne j$. Furthermore, the non-convex norm constraint \eqref{eq:cc_norm} causes Problem \ref{prob:cc_riemann} to be an MIBLP, which are in general more challenging to solve than mixed integer convex programs. This can cause solving Problem \ref{prob:cc_riemann} to be slow. To address this, we propose two reformulations:

\begin{itemize}
	\item In the first, we still optimize over the $N_\textrm{box}$ boxes simultaneously, and simply replace the non-convex constraint \eqref{eq:cc_norm} with a linear approximation: $\sum_i t_i \ge 1 - \varepsilon$. However, by doing this, the original chance-constraint $\pr(\trajxu\safe) \ge 1-\varepsilon$ may not hold exactly. One can get around this by instead enforcing $\sum_i t_i \ge 1 - \tilde\varepsilon$, incrementally shrinking $\tilde\varepsilon$ until the resulting trajectory satisfies the original chance constraint, but this can be cumbersome.
	\item The second reformulation, which is what we use in practice to solve Problem \ref{prob:cc_riemann}-$\varepsilon_\textrm{min}$, instead optimizes over the boxes one at a time. That is, if we are given a budget of $N_\textrm{box}$ boxes, we solve Problem \ref{prob:cc_riemann}-$\varepsilon_\textrm{min}$ for $N_\textrm{box} = 1$, then at the next iteration, solve for $N_\textrm{box} = 2$, where the first box is fixed. This continues until we reach the box budget. If $N_\textrm{box}$ is chosen to be large enough that the covered probability mass is the same when solving Problem \ref{prob:cc_riemann} for $N_\textrm{box}$ and $N_\textrm{box}+1$, this sequential strategy is lossless, i.e. will return the same solution as Problem \ref{prob:cc_riemann} without relaxations, for the case where choosing to satisfy one possible constraint can never cause the trajectory to not be able to satisfy a different constraint. In other cases, this strategy may lead to convergence to a local optimum in the amount of probability mass covered, but we observe that this strategy works well in practice.
\end{itemize}

\subsection{Sampling-based planners}\label{app:sampling_based_planners}

We utilize two sampling-based planners to compute open-loop plans in the case where the constraints of Problem \ref{prob:cc_riemann} are not representable in a mixed integer convex program: Minimum Constraint Removal (MCR) and the Blindfolded Traveler's Problem (BTP).

\textbf{Minimum Constraint Removal (MCR)} \cite{MCR}: MCR takes a \textit{finite} set of constraints and a start/goal state. At each step of the algorithm, MCR keeps track of the path from the start to the goal that violates the least number of constraints so far. The algorithm initializes this with the straight-line edge between the start and goal states, and sets $k$, the minimum number of constraints that must be violated as the number of constraints that the start/goal state violate. Then, MCR incrementally grows a roadmap, where candidate expansions are limited based on the number of constraints the candidate edge will violate, and at each growth iteration, finds the path from the start to each vertex which violates the minimum number of constraints and updates the path to the goal which violates the least constraints. Every so often, we increase $k$ to allow for more constraints to be violated when expanding the roadmap. This is summarized in Section 4.2 of \cite{MCR}.

To use MCR to approximate Problem \ref{prob:cc_riemann}-$\varepsilon_\textrm{max}$, we sample $N_\textrm{sample}$ constraints from the belief $\{\theta_i\sim b(\theta)\}_{i=1}^{N_\textrm{sam}}$ as input to MCR, then run MCR as just described. The output of MCR will then be a path on the roadmap connecting the start and goal which violates the minimal number of \textit{sampled constraints}.
 
\textbf{The Blindfolded Traveler's Problem (BTP)} \cite{BTP}: The Blindfolded Traveler's Problem (BTP) can be modeled as a graph search problem. It is defined by a graph $G = (V, E, W, \state_0, \state_g)$ with vertices $V$, edges $E$, weights $C$, and start/goal state $\state_0$ and $\state_g$. Furthermore, each edge $e$ is invalid with probability $p(e) \in [0, 1]$. If an edge $e_{uv}$ from $u$ to $v$ is traversed, either the traversal is successful, and the agent ends up at vertex $v$, or the agent discovers that the edge is invalid $\eta_{uv} \in [0, 1]$ fraction of the way through, at which point the agent needs to turn back and return to vertex $u$; such a traversal attempt costs $2 \eta_{uv}C_{uv}$, where $c_\task(e_{uv})$ is the cost of traversing edge $uv$. The agent has a prior over the probability of edge validity and blockage, which can be updated based on the observations gained through traversing the graph. A solution to BTP is a policy which takes the start state, history of observations, and outputs an edge to traverse.

While computing an optimal policy for the BTP is NP-complete, it is possible to compute high-quality approximations, for example using the Collision Measure strategy (Section 5.1 of \cite{BTP}). This strategy approximates BTP by modifying the graph edge weights to penalize the log probability of the edges being unsafe, that is, edge weights are modified such that $\tilde c_{uv} = c_{uv} - \beta \log(p(e_{uv}\textrm{ safe})$, for some weights $\beta$. We then compute paths on the graph by running A* with the modified edge weights.

Specifically, to approximate Problem \ref{prob:cc_riemann}-R with BTP, we sample constraints $\{\theta_i\sim b(\theta)\}_{i=1}^{N_\textrm{sam}}$ from the belief $b(\theta)$ to approximate the probabilities $p(e_{uv}\safe)$; specifically, for each edge, we estimate $p(e_{uv}\textrm{ safe})$ as the fraction of sampled constraints which are violated.

\subsection{Priors $p(\theta)$ other than the uniform distribution}\label{app:planning_priors}
In Section \ref{sec:plan_open_loop}, we discuss how to integrate over a \textit{uniform prior} to optimize boxes over the probability density which can be embedded in an MISOCP for planning probabilistically safe trajectories. Here, we discuss a different prior which also satisfies the closed-form integrability and log-concave assumptions detailed in Sec. \ref{sec:plan_open_loop}: $p(\theta) \propto \prod_{i=1}^{|\cstate|} (\bar \cstate_i(\theta) - \underline \cstate_i(\theta))$, where $\bar\cstate_i$, $\underline\cstate_i$ denote the upper and lower bounds of the box in dimension $i$ of the constraint space. This prior places more probability mass on larger box constraints; hence, the behavior generated using this prior is more conservative. As a concrete example, see Fig. \ref{fig:priors} for trajectories solving Problem \ref{prob:cc_riemann} for the different priors, solved over a range of different start/goal states. Observe that the trajectories that use the weighted prior are more conservative. Generally, an investigation of other useful priors which satisfy our assumption of closed-form integrability is the subject of future work.

\begin{figure}[!htb]
        \centering
        \includegraphics[width=0.6\linewidth]{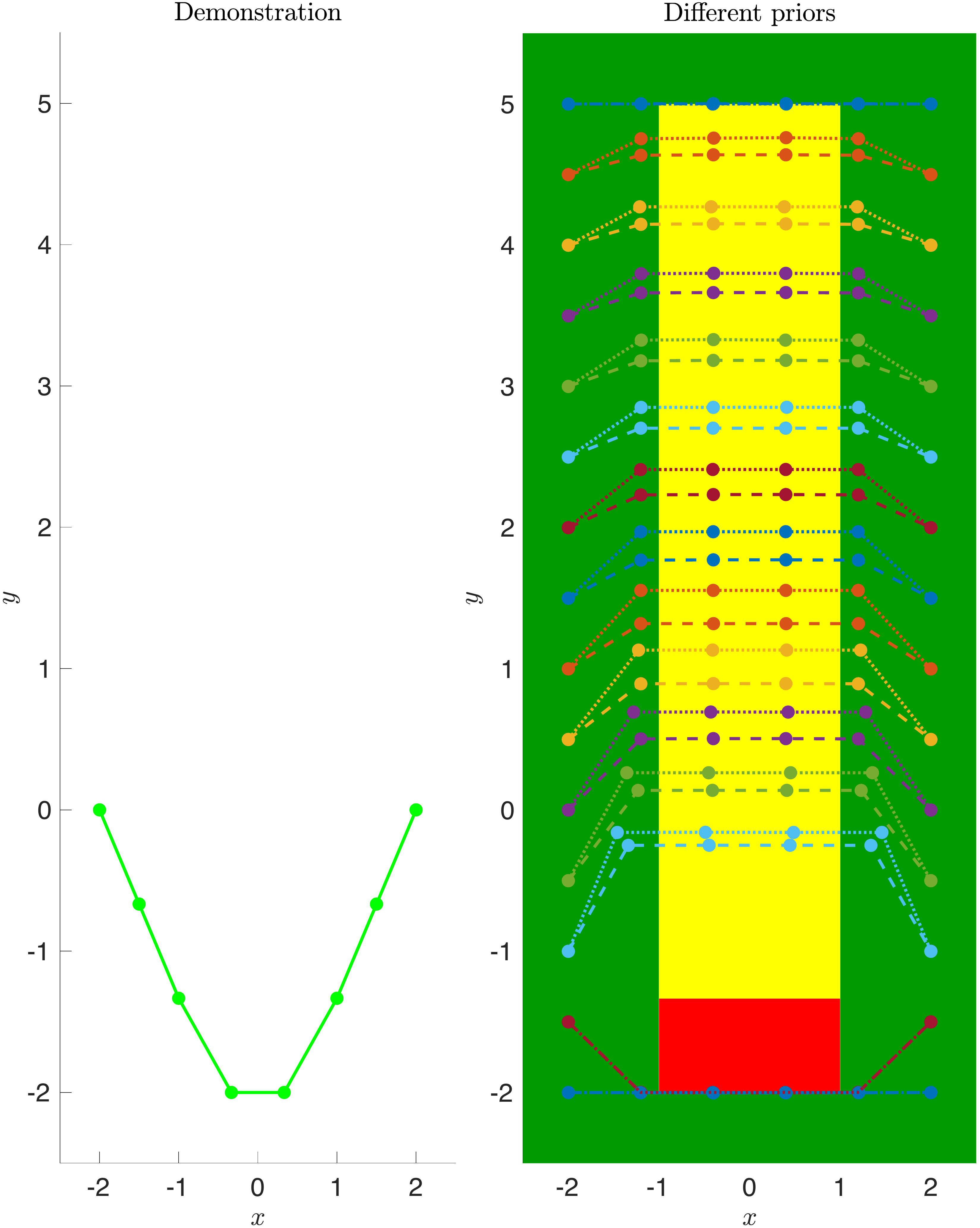}
        \caption{Trajectories generated for different priors. We are provided one demonstration (left), which reveals the left, right, and bottom extents of a box obstacle constraint, but not the upper extent. Dashed lines correspond to the uniform prior $p(\theta) \propto 1$, while dotted lines correspond to the prior $p(\theta) \propto \prod_{i=1}^{|\cstate|} (\bar \cstate_i(\theta) - \underline \cstate_i(\theta))$. }
        \label{fig:priors}
\end{figure}\vspace{-10pt}

\subsection{Belief updates}\label{app:planning_update}

Given an initial set of consistent constraint parameters $\feastheta$, we are interested in updating $\feastheta$ to be consistent with constraint information gathered in execution. We specifically write belief updates for the following constraint sensing modalities:

\textbf{Direct, exact measurements}: Consider a measurement that $\cstate_\textrm{safe}$ is safe, given an initial set of consistent constraints $\feastheta$. We want to find the maximum volume subset of $\feastheta$ which satisfies $g(\cstate, \theta) \le 0$. To accomplish this, we simply modify Problem \ref{prob:kkt_opt_robust_boxes} to:

\begin{equation*}\vspace{2pt}\label{eq:fwdprob}
	\hspace{-33pt}\begin{array}{>{\displaystyle}c >{\displaystyle}l >{\displaystyle}l}
				&\\[-13pt]
		\underset{s }{\text{maximize}} & \big(\textstyle\prod_i s_i\big)^{1/d} \\
		\text{subject to} & g(\cstate_\textrm{safe}, \theta + s \odot u) \le 0\\[1pt]
	\end{array}\hspace{-20pt}
\end{equation*}

where we can eliminate the uncertain variable $u$ with the identity used in Section \ref{sec:extraction_uob}, and use this modified problem in Algorithm \ref{alg:extraction}. Note that as the local optimality of the demonstrations is already embedded in the initial $\feastheta$, we do not need to add them as additional constraints in this modified problem, improving the computation time. 

The same modification can be done an unsafe measurement $\cstate_\textrm{unsafe}$, except with a constraint $g(\cstate_\textrm{unsafe}, \theta + s\odot u) > 0$.

\textbf{Ranged, exact measurements}: This case can come up when given LiDAR scans of the environment obtained in execution. For this setting, we assume that we are given a finite set of states which are all sensed to be safe, or all sensed to be unsafe. In our examples, we obtain this finite set of points by discretizing the possibly continuous ranged LiDAR measurement using a grid of measurement locations. This is a simple extension of the previously discussed modification for a single observed state; we simply have to add constraints corresponding to each state, and use this modified problem in Algorithm \ref{alg:extraction}:

\begin{equation*}\vspace{2pt}\label{eq:fwdprob}
	\hspace{-33pt}\begin{array}{>{\displaystyle}c >{\displaystyle}l >{\displaystyle}l}
				&\\[-13pt]
		\underset{s }{\text{maximize}} & \big(\textstyle\prod_i s_i\big)^{1/d} \\
		\text{subject to} & g(\cstate_\textrm{safe}^i, \theta + s \odot u) \le 0, \quad i = 1,\ldots,N_\textrm{safe}\\[1pt]
		 & g(\cstate_\textrm{unsafe}^i, \theta + s \odot u) > 0, \quad i = 1,\ldots,N_\textrm{unsafe}\\[1pt]
	\end{array}\hspace{-20pt}
\end{equation*}

\textbf{Direct, uncertain measurements}: In this case, suppose that we are given an initial set of consistent constraints $\feastheta$ as well as a finite set of states which may be possibly unsafe (the same ideas extend to the case where a set of states may be possibly safe); that is, we are given $\collset_{\neg s}$, where we learn in execution that at least one element of $\collset_{\neg s}$ is unsafe: $\exists \collset_{\neg s}^i \in \unsafeset(\theta^*)$. In our examples for the 7-DOF arm, we obtain this finite set of points by discretizing the continuous set of points which could be in contact by sampling points on the surface of the arm on the links downstream from where a torque limit is violated. Again, a similar modification can be made to Problem \ref{prob:kkt_opt_robust_boxes}:

\begin{equation*}\vspace{2pt}\label{eq:fwdprob}
	\hspace{-33pt}\begin{array}{>{\displaystyle}c >{\displaystyle}l >{\displaystyle}l}
				&\\[-13pt]
		\underset{s }{\text{maximize}} & \big(\textstyle\prod_i s_i\big)^{1/d} \\
		\text{subject to} & \bigvee_{i=1}^{|\collset_{\neg s}|} g(\collset_{\neg s}^i, \theta + s \odot u) \le 0\\[1pt]
	\end{array}\hspace{-20pt}
\end{equation*}

Specifically, the logical constraints over which state is unsafe can be modeled with binary variables, so the overall problem is still an MISOCP. The modified problem can be used in Algorithm \ref{alg:extraction}.

\section{Theory}\label{app:theory}

In this appendix, we provide proofs for the theorems in the main body of the paper.

\begin{theorem}
	If Alg. \ref{alg:extraction} terminates for any parameterization, its output is guaranteed to cover $\feastheta$.
\end{theorem}
\begin{proof}
	Suppose for contradiction that Algorithm \ref{alg:extraction} terminates such that $\feastheta \setminus (\bigcup_{i=1}^{N_\textrm{infeas}} \feasrobust^i) = \feastheta^\textrm{remain} \ne \emptyset$. However, by construction, Alg. \ref{alg:extraction} only terminates if there does not exist any $\theta \in \Theta$ for which $\{\textrm{KKT}(\demj)\}_{j=1}^{N_\textrm{dem}}$ can be satisfied; otherwise, Prob. \ref{prob:kkt_opt_robust_boxes} remains feasible. For all $\theta \in \feastheta^\textrm{remain}$, by definition of being an element of $\feastheta$, there exist $\lag$ to satisfy $\{\textrm{KKT}(\demj)\}_{j=1}^{N_\textrm{dem}}$. Contradiction.
\end{proof}

\begin{figure}[!htb]
        \centering
        \includegraphics[width=0.5\linewidth]{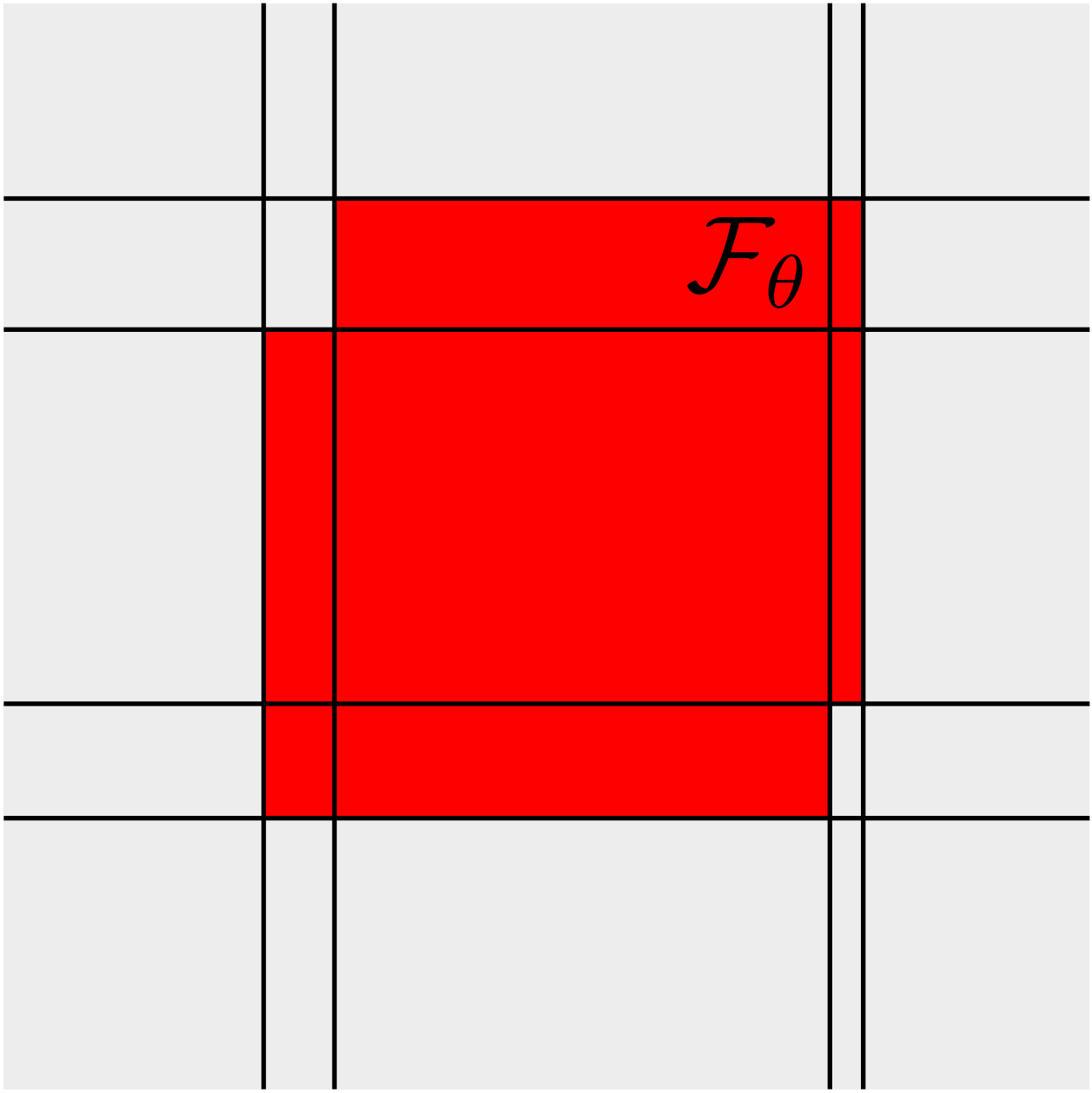}
        \caption{Grid used in the proof of Theorem \ref{thm:app_finite_time}.}
        \label{fig:grid}
\end{figure}

\begin{theorem}\label{thm:app_finite_time}
	Alg. \ref{alg:extraction} is guaranteed to terminate in finite time for union-of-boxes parameterizations.
\end{theorem}
\begin{proof}
	From Theorem \ref{thm:unions_of_boxes}, $\feastheta$ can be described as a union of a finite number of axis-aligned rectangles: $\feastheta = \bigcup_{i=1}^{N_\textrm{box}} \mathcal{B}_i \}$. Extend each hyperplane defining the boundary of a box $\mathcal{B}_i$ to infinity to obtain an irregular grid $\{\mathcal{G}_i\}_{i=1}^{N_\textrm{grid}}$ over $\Theta$ (see Figure \ref{fig:grid} for the case in 2D). As $\feastheta$ is composed of a finite number of boxes and hence there are a finite number of extended hyperplanes, there will be a finite number of grid cells, i.e. $N_\textrm{grid}$ is finite.
	
	We now prove that the solution of Problem \ref{prob:kkt_opt_robust_boxes} at any iteration $i$, $\feasrobust^i$, can be exactly represented by some subset of grid cells: $\feasrobust^i = \{\cstate \mid [I_{d\times d}, -I_{d\times d}]^\top \cstate \le [\bar \theta, -\underline \theta]^\top \} = \bigcup_{j=1}^{N_\textrm{rep}} \mathcal{G}_j$. Suppose for contradiction that there exists some grid cell $\mathcal{G}_k$ that $\feasrobust^i$ only partially contains: $(\mathcal{G}_k \cap \feasrobust^i \ne \emptyset) \wedge (\mathcal{G}_k \cap \feasrobust^i \ne \mathcal{G}_k)$. Formally, this means that in some coordinate of $\theta$, say the $m$th coordinate, the upper bound of $\feasrobust^i$, $\bar\theta_m^i$ satisfies $\bar\theta_m^i \in [\underline\theta_m^k, \bar\theta_m^k]$, where these denote the lower and upper bounds of grid $k$ in dimension $m$; similar logic holds for analyzing the lower bound. For $\bar\theta_m^i$ to be the upper bound of $\feasrobust^i$ in the $m$th coordinate, by the optimality of Problem \ref{prob:kkt_opt_robust_boxes}, there must exist some constraint state $\cstate$ contained in the expanded box $\{\cstate \mid [I_{d\times d}, -I_{d\times d}]^\top \cstate \le [\bar\theta_1,\ldots,\bar\theta_{m-1},\bar\theta_m^k,\bar\theta_{m+1},\ldots,\bar\theta_d, -\underline\theta]^\top \}$ such that $\cstate \notin \feastheta$. However, this is not possible, as by the grid partition, there exists no hyperplane defining $\feastheta$ that can be crossed in the $m$th coordinate between $\underline\theta_m^k$ and $\bar\theta_m^k$. Contradiction.
	
	Finally, as each iteration in Alg. \ref{alg:extraction} removes a finite number of grid cells, Alg. \ref{alg:extraction} will terminate in a finite number of iterations.
\end{proof}

\begin{theorem}
	A solution to Prob. \ref{prob:cc_riemann} is a guaranteed feasible, possibly suboptimal solution to Prob. \ref{prob:cc}.
\end{theorem}
\begin{proof}
	Feasibility follows by construction of Problem \ref{prob:cc_riemann}, as constraint \eqref{eq:cc_norm} directly models the probability constraint \eqref{eq:prob_safety}: $\int_{\mathcal{B}_i} d\theta = t_i^d$, so $\sum_i t_i^d = \sum_i \int_{\mathcal{B}_i} d\theta = \int_{\Theta_s} d\theta$ for $\Theta_s = \bigcup_{i=1}^{N_\textrm{box}} \mathcal{B}_i$, which is exactly $\pr(\trajxu\safe)$ when integrated over $\Theta_s$, for the uniform prior $b_\dem(\theta)$. 

	Suboptimality arises from the optimal partition of probability possibly not being representable as a union of boxes: in general, there exists $\Theta_s$ under which $c_\task(\trajxu)$ is minimized, such that there does not exist $N_\textrm{box}$ boxes where $\Theta_s = \bigcup_{i=1}^{N_\textrm{box}} \mathcal{B}_i$.
\end{proof}

\section{Further experimental details}\label{app:results}

In this section, we provide additional details on our experimental results. We first discuss in detail the baseline algorithms that we compare to in the results (Appendix \ref{app:planning_baselines}). We then demonstrate closed-loop planning with an uncertain nonlinear constraint using a sampled approximation of Problem \ref{prob:cc_riemann} (Appendix \ref{app:results_nl}), and then discuss additional details and visualize example runs of our method and baseline approaches for the mixed quadrotor uncertainty example (Appendix \ref{app:results_mixed}), the 7-DOF arm example (Appendix \ref{app:results_arm}), and the quadrotor maze example (Appendix \ref{app:results_maze}).

\subsection{Planning baselines}\label{app:planning_baselines}

\subsubsection{Mixed quadrotor example}\label{app:planning_baselines_mixed}

\textbf{Scenario approach}: The scenario approach \cite{scenario} satisfies uncertain constraints by sampling $N_\textrm{sam}$ possible constraint parameters $\{ \theta_i \}_{i=1}^{N_\textrm{sam}}$ and finds a solution that satisfies all of the sampled constraints. For this example, we may not be able to satisfy all of the sampled constraints, and hence the scenario approach may render the problem infeasible. To get around this to use the method as a baseline, we iteratively sample additional constraints, and solve the following open-loop planning problem:

\begin{subequations}
	\begin{align*}
    \hspace{-5pt}\min_{\trajxu }\quad\ & c_\task(\trajxu)\\[-4pt]
    \textrm{s.t.}\quad\  & \bar\phi(\trajxu) \in \bar\safeset \subseteq \bar\constraintspace,\ \phi_\task(\trajxu) \in \safeset_\task \subseteq \constraintspace_\task \\
	& \trajxu \in \safeset(\theta),\ \forall \theta \in \{ \theta_i \}_{i=1}^{N_\textrm{sam}}
\end{align*}
\end{subequations}

We stop sampling constraints when the problem becomes infeasible and return the feasible trajectory generated at the previous iteration.

\textbf{Optimistic approach}: In this approach, we are optimistic about the true constraint, only avoiding the set of guaranteed-unsafe states, buffering the extents by 0.5 in the uncertain constraint dimensions.

\subsubsection{7-DOF arm example}\label{app:planning_baselines_7dof}

\textbf{BTP without constraint parameterization or demonstrations}: In this version of BTP, we provide neither a union-of-boxes constraint parameterization nor a set of demonstrations. Instead, collision probabilities are measured with ``Collision Hypothesis Sets" (CHS) \cite{SaundB18}, which use a voxelization of the environment, with probabilities of particular voxels being occupied updated based on the occupancy of the robot volume during collision.

\textbf{Optimistic approach}: In this approach, we use the same graph provided to BTP and do not provide the information provided by the demonstrations, and we iteratively solve an optimistic problem. At the first iteration, we find run A* with all edges on the graph assumed valid, and attempt to execute the path. If the robot collides when traversing some edge $e_{uv}$ from vertex $u$ to $v$, the robot backtracks to vertex $u$, removes edge $e_{uv}$ from the graph, and replans on the modified graph. The procedure continues until the robot is at the goal.

\subsubsection{Quadrotor maze example}\label{app:planning_baselines_maze}

\textbf{Guaranteed-safe planning}: In this approach \cite{ral}, we compute paths which are guaranteed-safe with respect to the constraint uncertainty. Under the assumption that the constraint parameterization is correct, this guarantees that we will never need to replan upon discovering a constraint, but at the cost of possibly high-cost, conservative trajectories.

\textbf{Optimistic approach}: This approach \cite{janson} is optimistic with respect to the uncertain space, and constructs high-level plans that plan to intermediate goals on the frontier of unknown space, between the current state and the goal, and executes the best high-level plan. In more detail, we replicate the approximations used in Section IV.B of \cite{janson}. In particular, we assume that the unknown space is free, but buffer known obstacles by 0.1 meters in the uncertain dimensions. We discretize the unknown frontier in 2 meter intervals to construct our subgoals.

\subsection{Nonlinear constraint}\label{app:results_nl}
We show that our method can plan with constraint beliefs for non-union-of-boxes constraint parameterizations. Specifically, we are given a demonstration on a 2D kinematic system $[\chi_{t+1}, y_{t+1}]^\top = [\chi_t, y_t]^\top + [u_t^\chi, u_t^y]^\top$ which minimizes path length $c(\traj) = \sum_{t=1}^{T-1} \Vert x_{t+1} - x_t\Vert ^2$ (Figure \ref{fig:mcr}.A) while satisfying the constraint $g(\state, \theta) = \theta_1(\state_1^4 + \state_2^4) + \theta_2(\state_1^3 + \state_2^3) + \theta_3(\state_1 - 1)^3 + \theta_4(\state_2 + 1)^3 > 2$ for $\theta = [2, -5, 5, 5]^\top$. As the demonstration is rather uninformative, many $\theta$ make it locally-optimal. As the constraint is affine in $\theta$, we extract $\feastheta$ with zonotopes (running Algorithm \ref{alg:extraction} with Problem \ref{prob:kkt_opt_robust_zonotopes}); see Figure \ref{fig:mcr}.B-C for $\textrm{proj}_\statespace(\feastheta)$ and a corresponding probability heatmap. We now want to solve Prob. MCV, planning from $x_0 = [0, 0.75]^\top$ to $x_g = [0, -1.5]^\top$. As $g(\state, \theta)$ is not MICP-representable, we solve an approximation of Prob. MCV with samples, sampling 100 constraints from $b_\dem(\theta)$ as input to MCR (c.f. Section \ref{sec:plan_open_loop_samp}). The resulting path, Plan 1, (Figure \ref{fig:mcr}.C) violates one possible constraint at $\state_{\neg s} = [0.49, 0.8]^\top$. Though this path is safe for the true constraint, this is unknown to the learner, so we also precompute a contingency. Updating the belief $b_\textrm{ex}(\theta)$ with $\collset_{\neg s} = \{\state_{\neg s}\}$ and previously visited states as $\collset_s$ reduces the constraint uncertainty (Figure \ref{fig:mcr}.D-E). In particular, we note that the measurement at $[0.49, 0.8]^\top$ also removes the constraint uncertainty on the left-hand side of the space; this is due to the nonlinearity of the constraint parameterization, and as we can see in the belief (Figure \ref{fig:mcr}.C), only some of the the convex obstacles that cover the right side of the space can explain the possible collision at $[0.49, 0.8]^\top$. Hence, this measurement will update the belief to eliminate the non-convex obstacles, removing the left-side uncertainty. For this updated belief, Contingency 1 can be planned with no constraint violations (Figure \ref{fig:mcr}.E). Extracting $\feastheta$ with Algorithm \ref{alg:extraction} and planning with MCR takes 30 min. (can be sped up with a parallel implementation, c.f. Appendix \ref{app:extraction}) and 30 sec., respectively. 

\vspace{-10pt}
\begin{figure}[!htb]
        \centering
        \includegraphics[width=\linewidth]{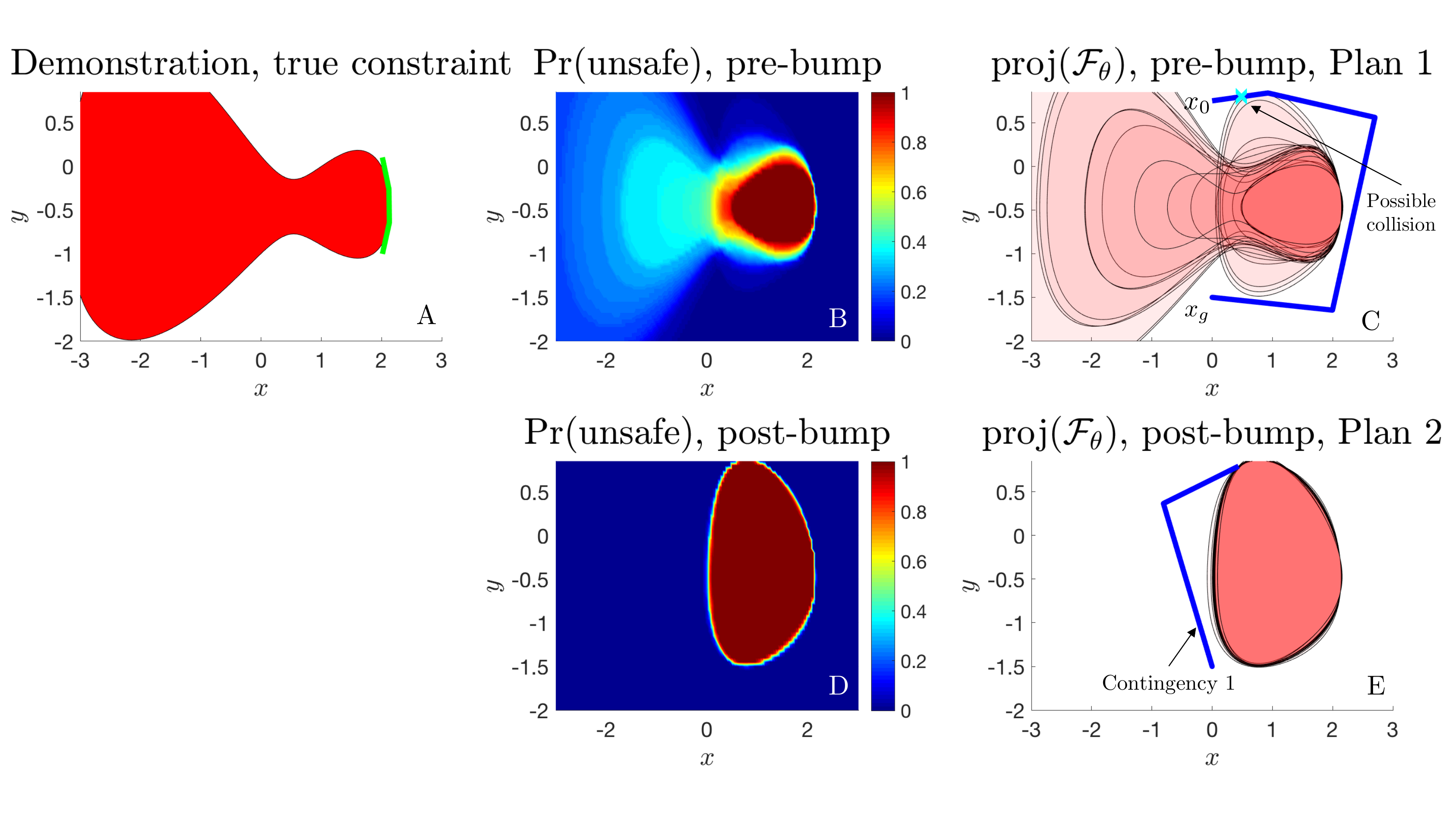}\vspace{-10pt}
        \caption{\hspace{-1pt}Example demonstrating MCR on planning with a nonlinear constraint. \textbf{A}: Demonstration, overlaid with the true constraint (red). \textbf{B}: Initial constraint uncertainty, visualized as a normalized probability heatmap.\textbf{C}. The initial plan (blue) generated by MCR, overlaid by a subsampling of 20 of the sampled constraint parameters $\theta$ provided to MCR. A possible collision occurs at the cyan ``x". \textbf{D}. The updated constraint uncertainty probability heatmap were the cyan state to be in collision. \textbf{E}. The new plan for the updated constraint uncertainty reaches the goal without violating any possible sampled constraints. }
        \label{fig:mcr}
\end{figure}\vspace{-11pt}

\subsection{Mixed state-control constraint uncertainty on a quadrotor}\label{app:results_mixed}

The system dynamics for the quadrotor \cite{quad_kth} are:
\begin{equation}
	\hspace{-5pt}\begin{bmatrix} \dot\chi \\ \dot y \\ \dot z \\ \dot\alpha \\ \dot\beta \\ \dot\gamma \\ \ddot \chi \\ \ddot y \\ \ddot z \\ \ddot \alpha \\ \ddot \beta \\ \ddot \gamma \end{bmatrix} = \begin{bmatrix} \dot\chi \\ \dot y \\ \dot z \\ \dot\beta \frac{\sin(\gamma)}{\cos(\beta)} + \dot\gamma \frac{\cos(\gamma)}{\cos(\beta)} \\ \beta \cos(\gamma) - \dot\gamma \sin(\gamma) \\ \dot\alpha + \dot\beta\sin(\gamma)\tan(\beta)+\dot\gamma\cos(\gamma)\tan(\beta) \\ -\frac{1}{m}[\sin(\gamma)\sin(\alpha) + \cos(\gamma)\cos(\alpha)\sin(\beta)]u_1 \\ -\frac{1}{m}[\cos(\alpha)\sin(\gamma) - \cos(\gamma)\sin(\alpha)\sin(\beta)]u_1 \\ g-\frac{1}{m}[\cos(\gamma)\cos(\beta)]u_1 \\ \frac{I_y-I_z}{I_x} \dot\beta \dot\gamma + \frac{1}{I_x}u_2 \\ \frac{I_z-I_x}{I_y} \dot\alpha \dot\gamma + \frac{1}{I_y}u_3\\ \frac{I_x-I_y}{I_z} \dot\alpha \dot\beta + \frac{1}{I_z}u_4\end{bmatrix},
\end{equation}
We time-discretize the dynamics by performing forward Euler integration with discretization time $\delta t = 0.7$. The 12D state is $x = [\chi, y, z, \alpha, \beta, \gamma, \dot x, \dot y, \dot z, \dot \alpha, \dot \beta, \dot \gamma]^\top$, and the relevant constants are $g = -9.81 \textrm{m}/\textrm{s}^2$, $m=1$kg, $I_x = 0.5\textrm{kg}\cdot\textrm{m}^2$, $I_y = 0.1\textrm{kg}\cdot\textrm{m}^2$, and $I_z = 0.3\textrm{kg}\cdot\textrm{m}^2$.

The simplified double integrator model that we use to plan in Problem \ref{prob:cc_riemann} is as follows:

\begin{equation}
	\begin{bmatrix}
		\dot\chi \\ \dot y \\ \dot z \\ \ddot x \\ \ddot y \\ \ddot z
	\end{bmatrix} = \begin{bmatrix}0 & 0 & 0 & 1 & 0 & 0 \\ 0 & 0 & 0 & 0 & 1 & 0 \\ 0 & 0 & 0 & 0 & 0 & 1 \\ 0 & 0 & 0 & 0 & 0 & 0 \\ 0 & 0 & 0 & 0 & 0 & 0 \\ 0 & 0 & 0 & 0 & 0 & 0 \end{bmatrix} \begin{bmatrix}
		\chi \\ y \\ z \\ \dot x \\ \dot y \\ \dot z
	\end{bmatrix} + \begin{bmatrix} 0 & 0 & 0\\ 0 & 0 & 0\\ 0 & 0 & 0\\ 1 & 0 & 0\\ 0 & 1 & 0\\ 0 & 0 & 1 \end{bmatrix} \begin{bmatrix} u_1 \\ u_2 \\ u_3 \end{bmatrix} + \begin{bmatrix}	0 \\ 0 \\ 0 \\ 0 \\0 \\ g\end{bmatrix}
\end{equation}
which is then time discretized with $\delta t$ = 0.5, with $g = -9.81$m$/s^2$.

\begin{figure}[!htb]
        \centering
        \includegraphics[width=\linewidth]{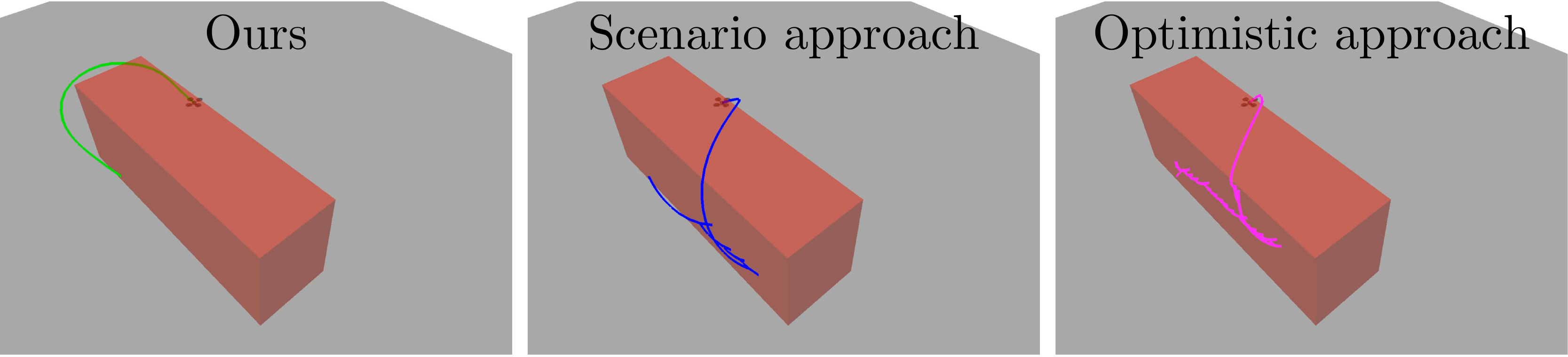}
        \caption{Example run 1 (mixed quadrotor example).}
        \label{fig:mixed_run1}
\end{figure}

\begin{figure}[!htb]
        \centering
        \includegraphics[width=\linewidth]{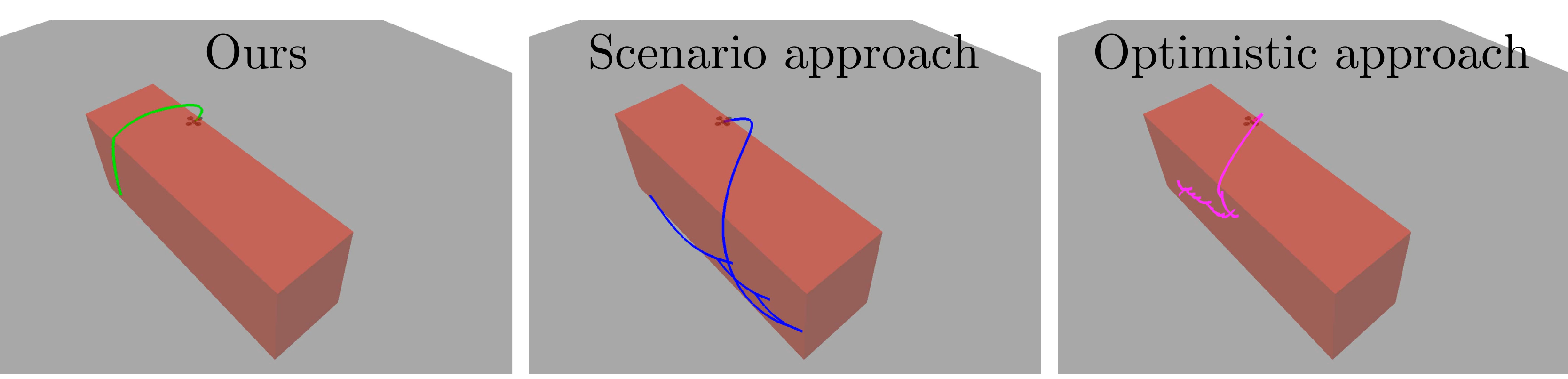}
        \caption{Example run 2 (mixed quadrotor example).}
        \label{fig:mixed_run2}
\end{figure}

\textbf{Visualizing example runs}: For two different sampled ground-truth constraints, we visualize the trajectories executed by our policy, the scenario approach policy \cite{scenario}, and the optimistic policy to compare their properties. 

In Figure \ref{fig:mixed_run1}, we display the sampled state constraint in red, and the sampled control constraint is $\Vert \control \Vert_2^2 \le 98.27$. Our policy, which attempts to satisfy all of the state constraints by trying to move above all of the possible obstacles (see Sec. \ref{sec:results} for more discussion), violates the control constraint at the first time-step by attempting to do so. Our policy then switches to the first contingency plan and successfully reaches the goal with 1 constraint violation. On the other hand, the scenario approach suffers 6 constraint violations (due to sampling constraints), while the optimistic approach suffers 25 constraint violations (because it does not try to avoid any states other than those which are known to be unsafe).

In Figure \ref{fig:mixed_run2}, we display the sampled state constraint in red, and the sampled control constraint is $\Vert \control \Vert_2^2 \le 98.55$. Our policy reaches the goal in one try (0 constraint violations), as we do not end up violating the control constraint. On the other hand, the scenario approach suffers 4 constraint violations, while the optimistic approach suffers 15 constraint violations.

\textbf{Histogram}: In computing Plan 1 (the initial executed trajectory) and Contingencies 1-4 (the trajectories we switched to upon observing 1, 2, 3, and 4 constraint violations), we can calculate the volumes of the covered $\mathcal{B}_i$ which are optimized by Problem \ref{prob:cc_riemann} at each iteration. When computing Plan 1, we cover $p_0 = 69.10\%$ of the possible constraints; when computing Contingency 1, we cover $p_1 = 54.90\%$ of the possible constraints under the updated belief, $p_2 = 56.33\%$ for Contingency 2, $p_3 = 73.91\%$ for Contingency 3, and $p_4 = 98.00\%$ for Contingency 4. Using these percentages, we can calculate the theoretical frequency of overrides as $p(0 \textrm{ overrides}) = p_0$, $p(1 \textrm{ override}) = (1-p_0)p_1$, and so on, until $p(4 \textrm{ overrides}) = \prod_{i=1}^3 (1-p_i)p_4$. Using these formulae, we compute the theoretical probability of suffering $i$ constraint violations before reaching the goal, which we compare with the empirical histograms (normalized over 500000 trials) (Fig. \ref{fig:hist}), and we see the statistics match quite closely.

\begin{figure}[!htb]
        \centering
        \includegraphics[width=0.45\linewidth]{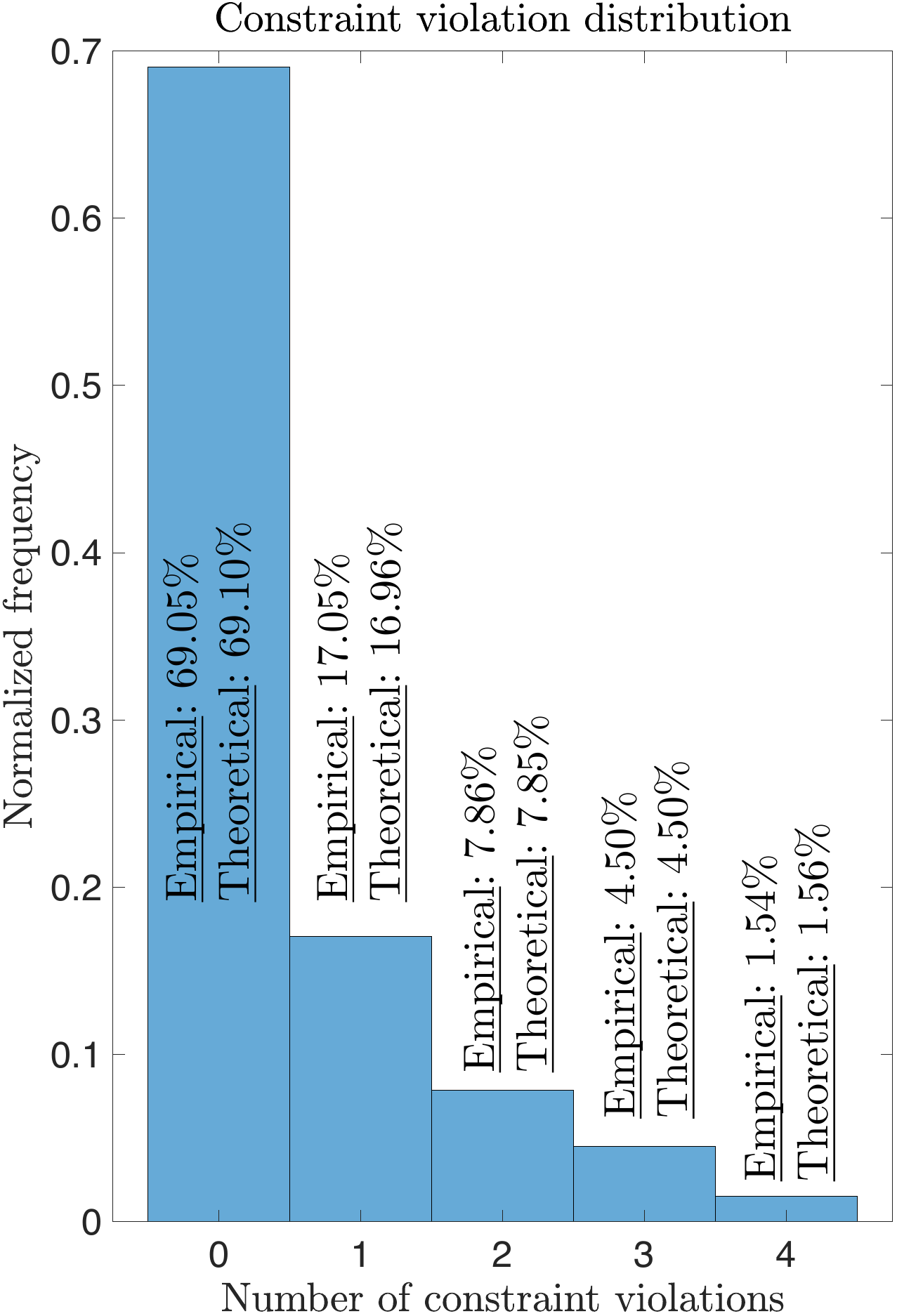}
        \caption{Constraint violation histogram for the mixed quadrotor uncertainty example.}
        \label{fig:hist}
\end{figure}\vspace{-10pt}

\subsection{7-DOF arm with contact sensing uncertainty}\label{app:results_arm}

We use a kinematic model of the arm: $\state_{t+1}^i = \state_t^i + u_t^i$, for $i = 1,\ldots,7$. Here, $\state \in \mathbb{R}^7$, where coordinate $i$ of the state denotes the angle of the $i$th joint. For planning, we use a BTP graph with 5000 vertices (c.f. Section \ref{sec:plan_open_loop_samp}). In the following, we discuss more details on the performance of different policies on the task discussed in Section \ref{sec:results}.

\textbf{BTP with CHS}: We present a time-lapse of the trajectory executed by running BTP with CHS (as described in Appendix \ref{app:planning_baselines_7dof}) in Figure \ref{fig:arm_btpchs}. Since this approach is not given demonstrations, it requires several bumps in order to sufficiently localize the shelf. Furthermore, this approach does not leverage a union-of-boxes constraint parameterization as a prior on the world, as the CHS does not extrapolate about the constraint beyond the sensed collision, making it so that more bumps occur before reaching the goal.

\begin{figure}[!htb]
        \centering
        \includegraphics[width=\linewidth]{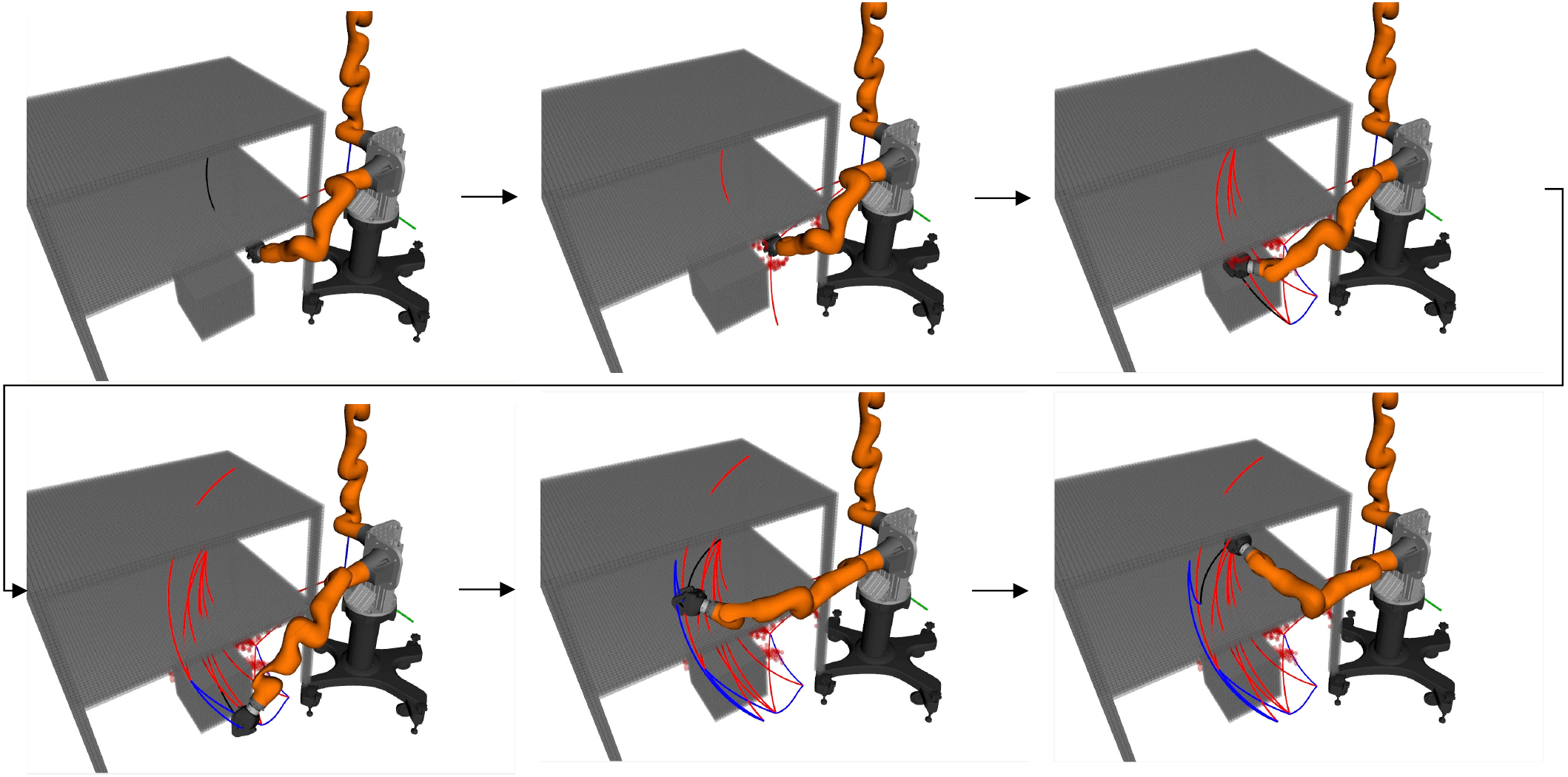}
        \caption{BTP with CHS. Voxels are colored red if they are possibly unsafe according to the CHS. Red edges are attempted edges which are discovered to be blocked in execution. The attempted edges which were unblocked are colored blue.}
        \label{fig:arm_btpchs}
\end{figure}

\textbf{Optimistic approach}: We present the trajectory executed by running the optimistic strategy described in Appendix \ref{app:planning_baselines_7dof} in Figure \ref{fig:arm_opt}. Since this strategy entirely ignores the correlation in validity between edges which are close to each other, it explores many edges, yielding a trajectory cost which is much higher than the other approaches.

\begin{figure}[!htb]
        \centering
        \includegraphics[width=0.4\linewidth]{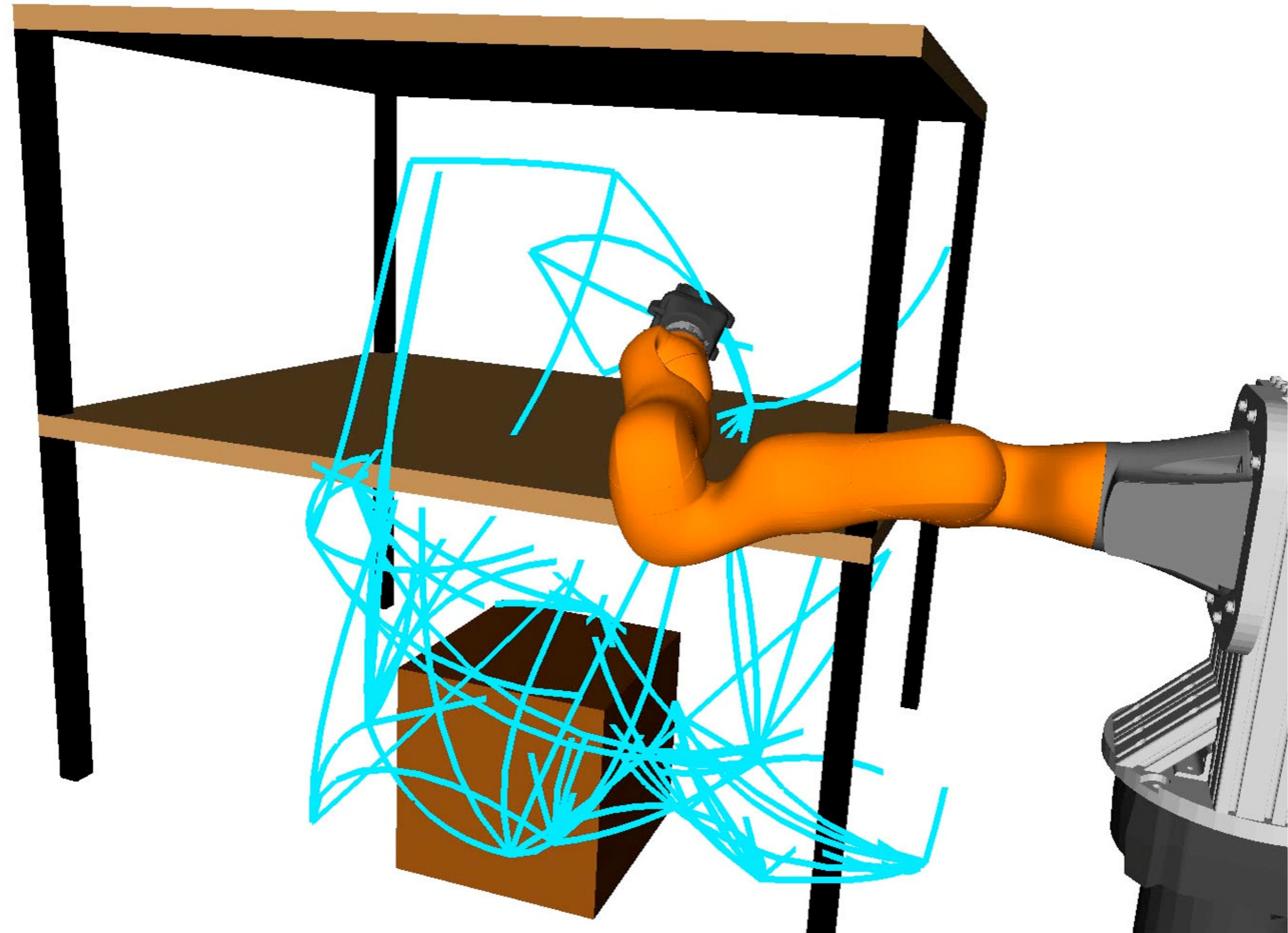}
        \caption{Trajectory executed by the optimistic policy.}
        \label{fig:arm_opt}
\end{figure}

\textbf{Suboptimal human demonstrations}: Finally, we present a time-lapse of our policy when initialized with suboptimal human demonstrations in Figure \ref{fig:arm_subopt}. The demonstrations are captured using an HTC Vive in a virtual reality simulation environment (Figure \ref{fig:arm_vive_demos}). Overall, the behavior of our policy when initialized with these human demonstrations is similar to the the case of synthetic demonstrations (which is discussed in Sec. \ref{sec:results}): it plans to move around the shelf, and in doing so, collides with the unmodeled obstacle. This triggers a constraint parameterization update, and the replanned path (which avoids the shelf and the uncertain region induced by the collision) steers the arm to the goal without further collision. For this case, the executed trajectory cost is 7.78 rad.

\begin{figure}[!htb]
        \centering
        \includegraphics[width=\linewidth]{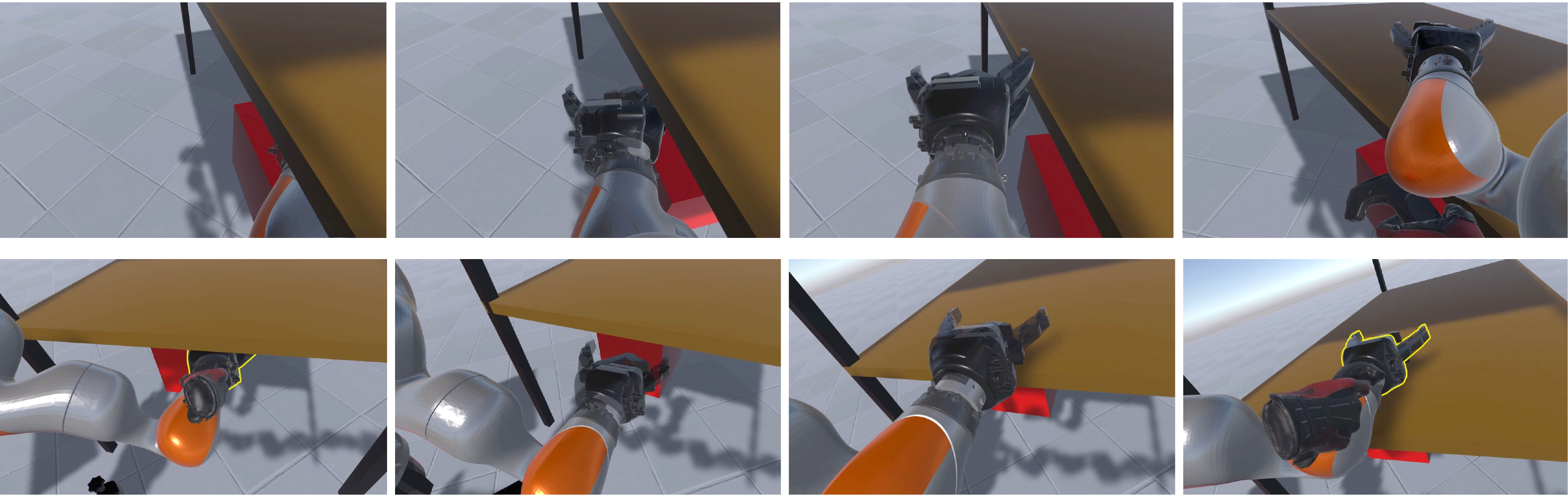}
        \caption{Suboptimal human demonstrations. Top row: time-lapse of the first demonstration. Bottom row: time-lapse of the second demonstration.}
        \label{fig:arm_vive_demos}
\end{figure}

\begin{figure}[!htb]
        \centering
        \includegraphics[width=\linewidth]{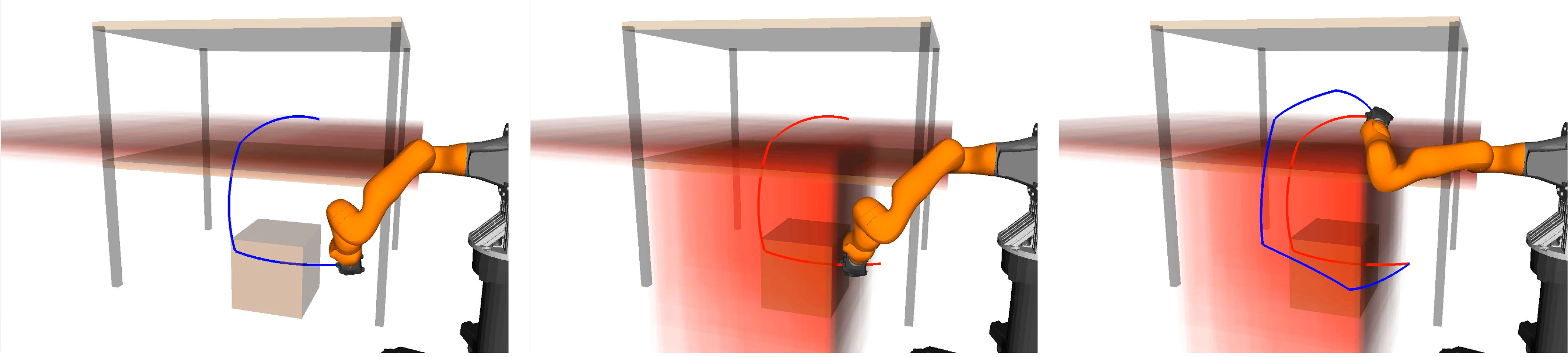}
        \caption{Policy when initialized with suboptimal demonstrations. Left: plans an initial trajectory that bumps into the unmodeled obstacle. Center: constraint parameterization is updated to two boxes. Right: replanned trajectory successfully avoids all collisions, steering the arm to the goal.}
        \label{fig:arm_subopt}
\end{figure}

\subsection{Quadrotor maze}\label{app:results_maze}

\textbf{Inevitable collision states}: We integrate the inevitable collision state \cite{ICS} avoidance constraints into our approach. We first enforce that on a planned trajectory, each state which may be possibly unsafe must be observable (within 2 meters) from a previous state on the trajectory. Furthermore, the line of sight between this previous state and the possible unsafe state cannot be occluded by any obstacle other than the obstacle which is making the state possibly unsafe. For these possibly unsafe states, we enforce that we can brake in time to avoid collision (and together with the line-of-sight constraint, enforces that we can also sense if we need to brake). This is done by explicitly optimizing ``brake trajectories" in conjunction with the planned trajectory, which are rooted two time-steps on the plan before a possible collision and which bring the system to a stop without violating any possible constraints. We further enforce the line of sight constraint by discretizing the line segment between $\state_t$ and $\state_{t+2}$ (if $\state_{t+2}$ is possibly unsafe) into 10 points, and enforcing that each discretized state is guaranteed to satisfy all constraints other than the uncertain constraint which can make $\state_{t+2}$ possibly unsafe.

\begin{figure}[!htb]
        \centering
        \includegraphics[width=0.66\linewidth]{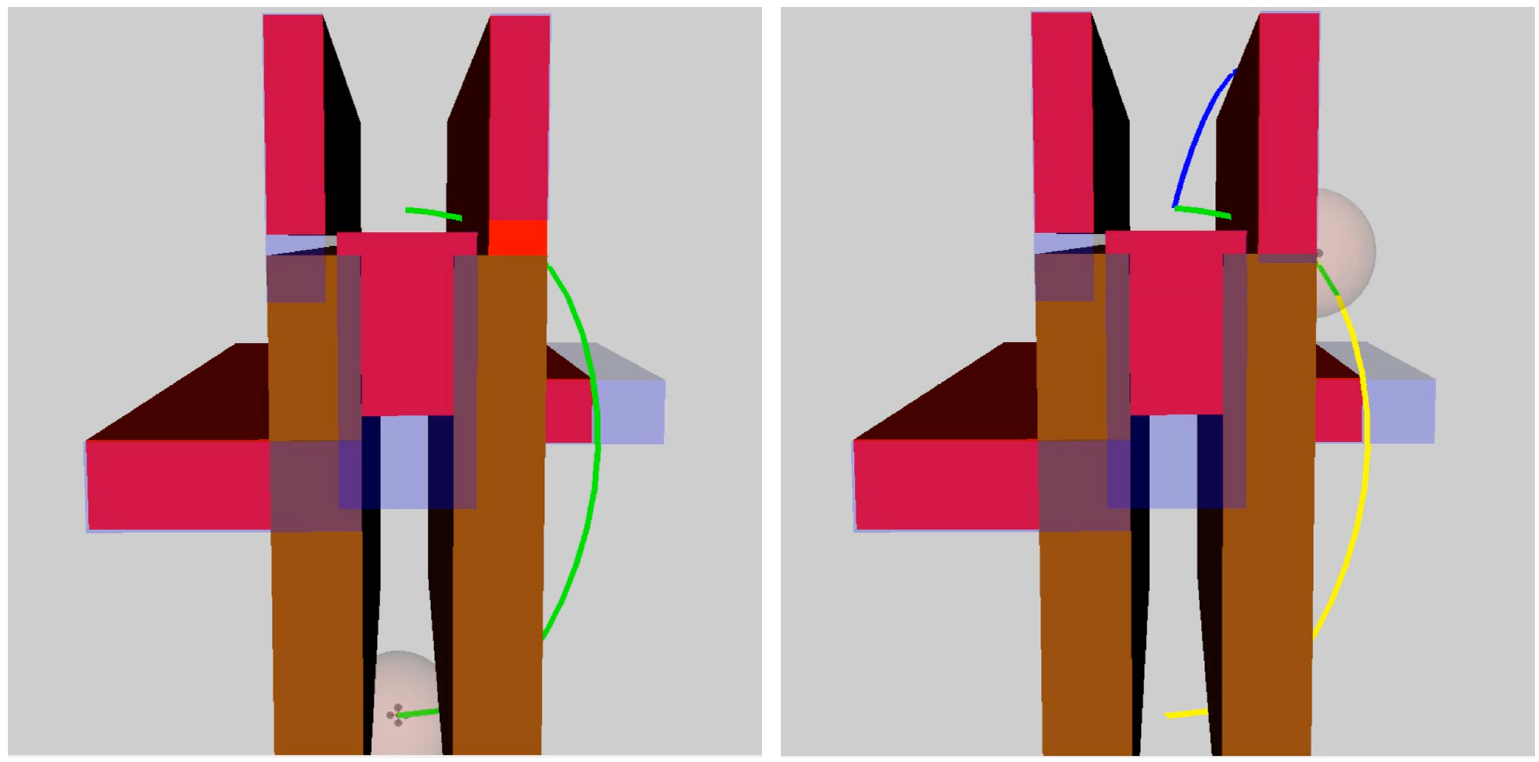}
        \caption{Example run, our policy (quadrotor maze). Initial plan (green), contingency plan (blue), actually executed plan (yellow). The sphere around the quadrotor indicates the sensing radius.}
        \label{fig:maze_run_ours}
\end{figure}

\begin{figure}[!htb]
        \centering
        \includegraphics[width=0.33\linewidth]{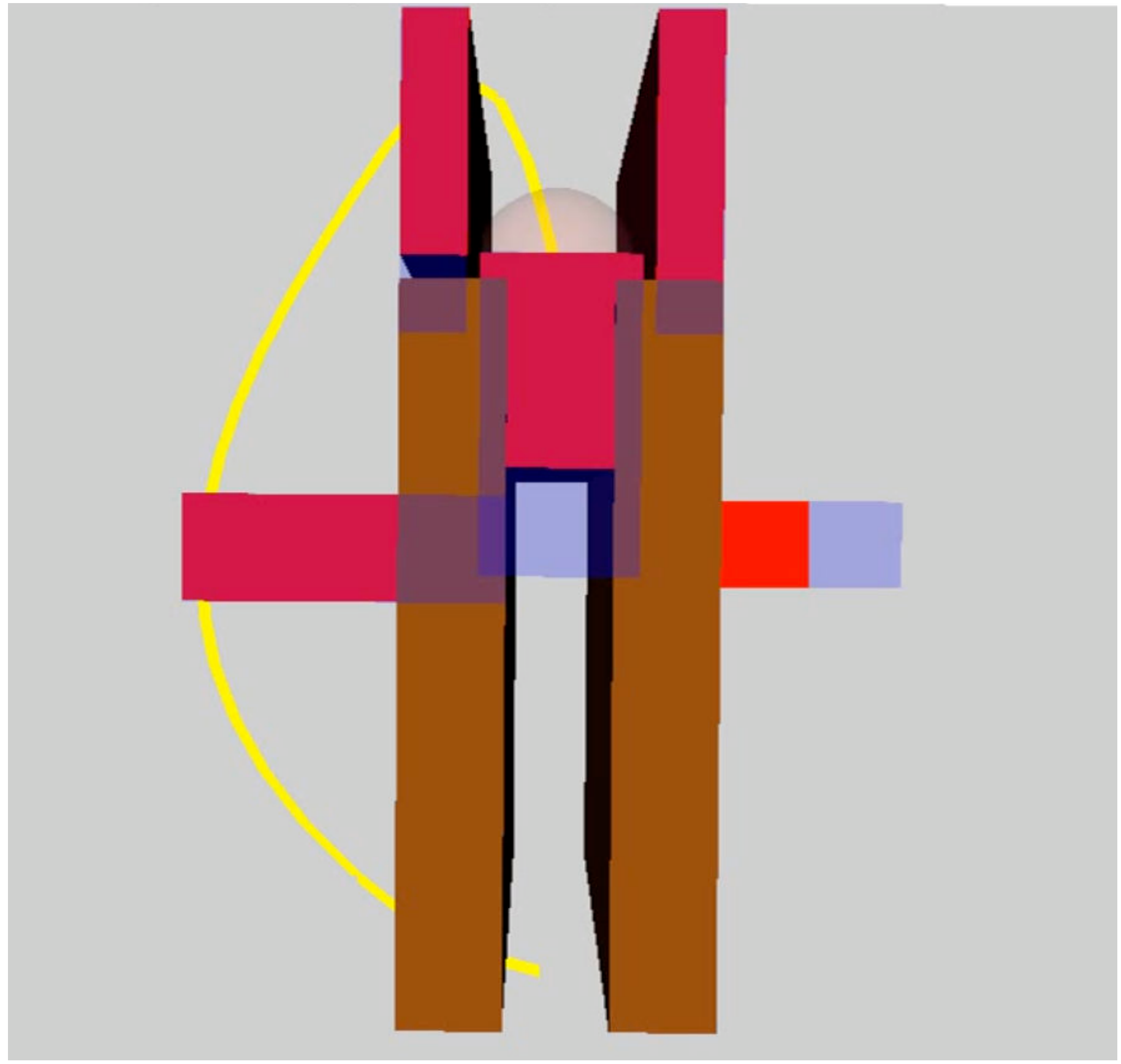}
        \caption{Example run, guaranteed-safe policy (quadrotor maze). Initial/actually executed plan (yellow).}
        \label{fig:maze_run_ral}
\end{figure}

\begin{figure}[!htb]
        \centering
        \includegraphics[width=\linewidth]{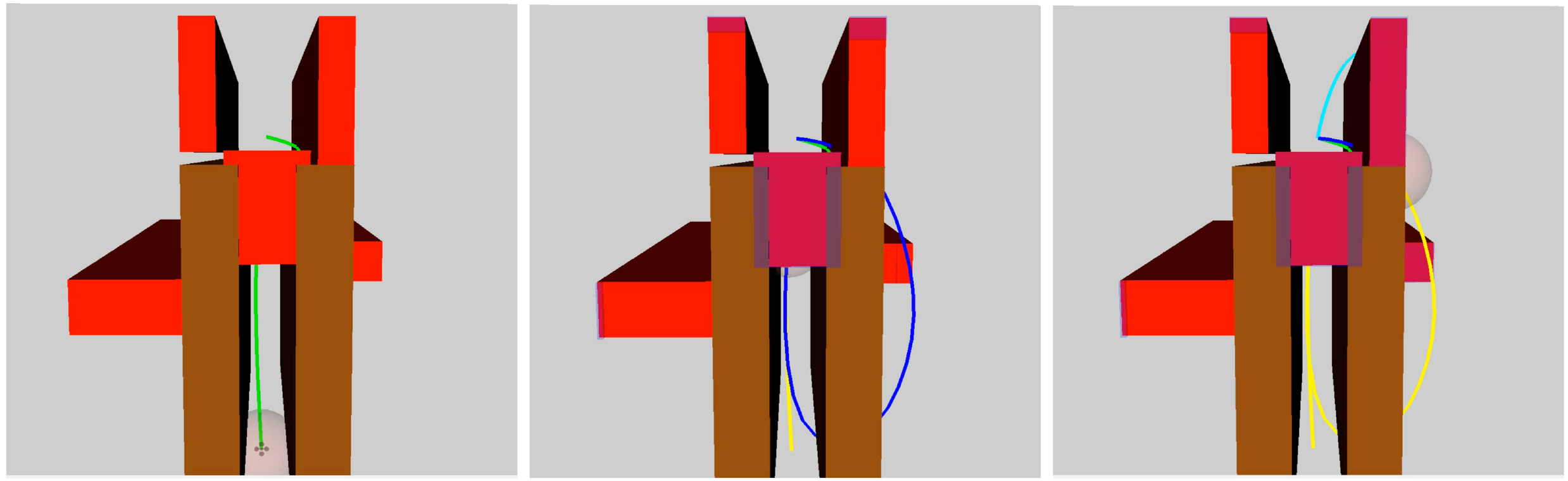}
        \caption{Example run, optimistic policy (quadrotor maze). Initial plan (green), contingency plan 1 (blue), contingency plan 2 (cyan), actually executed plan (yellow).}
        \label{fig:maze_run_pavone}
\end{figure}

\textbf{Visualizing an example run}: For one sampled possible environment (displayed in red in Figures \ref{fig:maze_run_ours}-\ref{fig:maze_run_pavone}), we visualize the trajectories executed by our policy, a policy which seeks to plan guaranteed-safe trajectories \cite{ral}, and an optimistic policy \cite{janson}. Our policy moves to the right and seeks to cut through the possibly unsafe region in the top right (Figure \ref{fig:maze_run_ours}); when observing that the region is blocked, our policy switches to the blue contingency trajectory. On the other hand, the guaranteed-safe policy (Figure \ref{fig:maze_run_ral}) seeks to avoid all possible constraints; as a result, while this policy never needs to switch to a contingency plan, it also ends up deterministically executing a higher-cost trajectory. Finally, the optimistic policy in \cite{janson} explores the dead end between the brown obstacles and is forced to backtrack, yielding a higher cost compared to our policy.

\subsection{Computation times}\label{app:computation_time}

One challenge that our method faces when applied to real-time replanning is the computational intensity of online belief updates and online replanning of open-loop trajectories. First, we emphasize that if the assumptions in Section \ref{sec:plan_policy} hold, we can precompute the possible belief updates and contingency plans to avoid computing them online. If the assumptions are not satisfied, we will need to perform the computation online.

For belief updates, the computation time and number of measurements that can be updated depends heavily on the measurement type. For instance, updating the belief on the quadrotor maze example for a LiDAR scan with 30000 discretized points takes 1.4 seconds; LiDAR-type measurements are fast as each point is known safe or unsafe. However, contact measurements (as seen in the 7-DOF arm examples) are expensive as an unknown combination of the discretized points can be unsafe; modeling this requires the addition of many binary decision variables; it takes 30 minutes to incorporate a contact measurement with 300 discretized points. In this case, a further investigation of the tradeoff between accuracy and computation time based on the number of sampled possible contact measurements may lead to further computational gains.

The integer optimization variables are the key reason for slow $\feastheta$ extraction in Algorithm \ref{alg:extraction}. Thus, we are optimistic that we can speed up computation with parallelization (see Appendix \ref{app:parallel}) and recent advances in fast mixed integer programming \cite{mip_ms}, which enjoy orders of magnitude speedup by learning efficient branching heuristics. 

For open-loop planning, we note that the aforementioned fast mixed integer programming methods, as well as other work in warm-starting mixed integer programs, can be useful in reducing planning times for solving Problem \ref{prob:cc_riemann} and other variants, as all of these variants are mixed integer programs, and the previous open-loop plan can serve as a good initalization for replanning. We also emphasize that we can reduce BTP planning times for the 7-DOF arm examples to around 15 seconds (as in the original BTP paper \cite{BTP}) by precomputing arm swept volumes along roadmap edges and by employing lazy collision checking.

\end{document}